\documentclass[11pt]{article}

\usepackage{amsthm}
\usepackage{amssymb}
\usepackage{amsmath}
\usepackage{graphicx}
\usepackage{algorithm}
\usepackage{algorithmic}
 \usepackage{tcolorbox}
 \usepackage{geometry}
 \usepackage{bm}
 \usepackage[titletoc,title]{appendix}

\newtheorem{theorem}{Theorem}
\newtheorem{lemma}{Lemma}

\usepackage[hidelinks]{hyperref}
\usepackage[numbers,sort&compress]{natbib}

\newcommand{\longversion}[1]{}

\renewcommand{\longversion}[1]{}

\newcommand{\EL}{\mathcal{L}}
\newcommand{\MM}{\min_{x\in X}\max_{y\in Y}}
\newcommand{\MMD}{\min_{x\in \Delta}\max_{y\in \Delta}}
\newcommand{\MMDt}{\min_{x\in \Delta_\theta}\max_{y\in \Delta_\theta}}

\title{Competing Against Equilibria in Zero-Sum Games with Evolving Payoffs\footnote{This is an extended version of \cite{cardoso2019competing}, a new section on training GANs has been added. }}

\makeatletter
\renewcommand\@date{{%
  \vspace{-\baselineskip}%
  \large\centering
  \begin{tabular}{@{}c@{}}
    Adrian Rivera Cardoso  \textsuperscript{1} \\
  \end{tabular}%
  \begin{tabular}{@{}c@{}}
    , Jacob Abernethy \textsuperscript{2} \\
  \end{tabular}
    \begin{tabular}{@{}c@{}}
    , He Wang \textsuperscript{1} \\
  \end{tabular}
    \begin{tabular}{@{}c@{}}
    , Huan Xu \textsuperscript{1} \\
  \end{tabular}

  \bigskip

  \textsuperscript{1}Department of Industrial and Systems Engineering, Georgia Institute of Technology\par
  \textsuperscript{2}Department of Computer Science, Georgia Institute of Technology

  \bigskip

  \today
}}
\makeatother

\begin{document}

\maketitle

\begin{abstract}
We study the problem of repeated play in a zero-sum game in which the payoff matrix may change, in a possibly adversarial fashion, on each round; we call these Online Matrix Games. Finding the Nash Equilibrium (NE) of a two player zero-sum game is core to many problems in statistics, optimization, and economics, and for a fixed game matrix this can be easily reduced to solving a linear program. But when the payoff matrix evolves over time our goal is to find a sequential algorithm that can \textit{compete with}, in a certain sense, the NE of the long-term-averaged payoff matrix. We design an algorithm with small NE regret--that is, we ensure that the long-term payoff of both players is close to minimax optimum in hindsight. Our algorithm achieves near-optimal dependence with respect to the number of rounds and depends poly-logarithmically on the number of available actions of the players. Additionally, we show that the naive reduction, where each player simply minimizes its own regret, fails to achieve the stated objective regardless of which algorithm is used. We also consider the so-called \textit{bandit setting}, where the feedback is significantly limited, and we provide an algorithm with small NE regret using one-point estimates of each payoff matrix.
\end{abstract}

\section{Introduction}
We consider a problem in which two players interact in a zero-sum game repeatedly. The payoff matrix of the game is unknown to the players \emph{a priori}, and may change arbitrarily on each round. Our objective is to find competitive strategies that can achieve the Nash equilibrium of the game with the average payoffs in the long term. This problem is a significant extension of the classical learning setting in zero-sum games, where the underlying payoff matrix is often assumed to be fixed or i.i.d. In contrast, we allow the payoff matrix to evolve arbitrarily in each round, and can even be selected in a possibly adversarial fashion. 

Zero-sum games \cite{neumann1928zerlegung,morgenstern1953theory} are ubiquitous in economics and central to understanding Linear Programming duality \cite{hazan2016introduction,adler2013equivalence}, convex optimization  \cite{abernethy2017frank,abernethy2018faster}, robust optimization \cite{ben2009robust}, and Differential Privacy \cite{dwork2014algorithmic}. 
The task of finding the Nash equilibrium of a zero-sum game is also connected to several machine learning problems such as: Markov Games \cite{littman1994markov}, Boosting \cite{freund1996game}, Multiarmed Bandits with Knapsacks \cite{badanidiyuru2013bandits,immorlica2018adversarial} and dynamic pricing problems \cite{ferreira2017online}.

We formally define the problem setting in Section~\ref{sec:problem-formulation}. We then 
highlight the main contributions of this paper in Section~\ref{sec:contributions} and discuss related works in Section~\ref{sec:liter-rev}.


\subsection{Problem Formulation: Online Matrix Games}
\label{sec:problem-formulation}
We start by reviewing the definition of classical two-player zero-sum games. Suppose player 1 has $d_1$ possible actions and player 2 has $d_2$ possible actions. The payoffs for both players are determined by a matrix $A\in \mathbb{R}^{d_1\times d_2}$, with $A_{i,j}$ corresponding to the loss of player 1 and the reward of player 2 when they choose to play actions $(i,j)\in[d_1]\times[d_2]$.\footnote{Throughout, $[n]\triangleq\{1,...,n\}$ for any positive integer $n$.} We allow the players to use \emph{mixed strategies} -- each mixed strategy is represented by a probability distribution over their actions.
More specifically, when Player 1 uses a mixed strategy $x \in \Delta_{d_1}$ and Player 2 uses a mixed strategy  $y\in \Delta_{d_2}$, the expected payoff is $x^\top A y$.\footnote{Here, $\Delta_{d}$ represents the unit simplex in dimension $d$: $ \Delta_{d} \triangleq\{v\in \mathbb{R}^{d}: \Vert v\Vert_1 =1 , v\geq 0\}$.}
Throughout the paper, we refer to the static zero-sum game as a \emph{matrix game} (MG), because the players' payoffs are a bilinear function encoded by the matrix $A$. 
A Nash equilibrium of this game is defined as any pair of (possibly) mixed strategies $(x^*, y^*)$ such that 
$$ (x^*)^\top A y \leq  (x^*)^\top A y^* \leq x^\top A y^*$$
for any $x \in \Delta_{d_1}, y\in \Delta_{d_2}$.
It is well known that every MG has at least one Nash equilibrium \cite{morgenstern1953theory}. 
The problem of finding an equilibrium for a MG can be reduced to solving linear programming problems. In fact, \cite{adler2013equivalence} showed that the opposite is also true, every linear programming problem can be solved by finding an equilibrium to a corresponding MG. 

Now, we define a problem that generalizes the matrix games into an online setting, which we call the Online Matrix Games (OMG) problem.
Suppose two players interact in a repeated zero-sum matrix game through $T$ rounds. 
In every round $t\in[T]$, they must each choose a (possibly) mixed strategy from the given action sets $x_t \in \Delta_{d_1}, y_t \in \Delta_{d_2}$. However, we assume that the payoff matrix in OMG can evolve in each round, and
the players have no knowledge of the payoff quantities in that round before they commit to an action. 
Let $\{A_t\}_{t=1}^T$ be an arbitrary sequence of matrices, where each $A_t \in [-1,1]^{d_1\times d_2}$ for all $t=1,..,T$. For each round $t$, the players 
choose their mixed strategies $x_t \in \Delta_{d_1}, y_t \in \Delta_{d_2}$ before the matrix $A_t$ is revealed. Then, player 1 (resp. player 2) receives a loss (resp. gain) given by the payoff quantity $x_t^{\top} A_t y_t$.
Note that the payoff matrix $A_t$ is allowed to change arbitrarily from round to round and may even depend on the past actions of both players.
The \textit{joint goal} for both players is to find strategies that ensure their average payoffs in $T$ rounds is close to the Nash Equilibrium under the average payoff matrix $\frac{1}{T}\sum_{t=1}^T A_t$ in hindsight.

More precisely, let us call the quantity 
\begin{equation}\label{eq:NE-regret}
\left| \sum_{t=1}^T x_t^{\top}A_t y_t - \MM  \sum_{t=1}^T x^{\top}A_t y \right|
\end{equation}
the \textit{Nash Equilibrium (NE) regret}. This is a natural extension of the regret concept in typical online learning or multi-armed bandit problems, which involve only a single decision maker.
The primary objective of the OMG problem is to find online strategies for both players so that, as $T \to \infty$, {the average NE regret \eqref{eq:NE-regret} per round tends to 0 (i.e., the NE regret is $o(T)$)}.

We make some remarks about the choice of benchmark and the fact that the players must update jointly despite the fact that they are playing a zero-sum game. In the following examples, the comparator term $\MM  \sum_{t=1}^T x^{\top}A_t y $ arises naturally and there is one decision maker which chooses the actions of both players.
\begin{enumerate}
\item  Online Linear Programming \cite{agrawal2014dynamic}: the decision maker solves an LP where data arrives sequentially. This problem has real-world applications in ad-auctions. Using Lagrangian duality, we can reduce this problem to an online zero-sum game (our setting), where player 1 chooses primal variables and player 2 chooses dual variables. Our benchmark corresponds to the optimal solution of the offline LP.
\item Adversarial Bandits with Knapsacks \cite{immorlica2018adversarial}: this problem extends the classical Multi Armed Bandit by adding a `knapsack' constraint. Again, using a Lagrangian relaxation on the knapsack constraint, this problem can be linked to the online min-max games that we study (see Sec. 3.2 of \cite{immorlica2018adversarial}).
\item Generative Adversarial Networks \cite{goodfellow2014generative}: GANs can also be viewed as a zero-sum game, where the decision maker trains the generator and discriminator to find a Nash equilibrium. Although our model cannot directly be used for GANs because they are nonconvex, it is another example where both players may desire to update jointly. In Section \ref{sec:GANs} we explore this further.
\end{enumerate}

In the paper, we consider the OMG problem in two distinct information feedback settings. In the \emph{full information} setting (Section~\ref{sec:full-info}), both players are able to observe the full matrix $A_t$ at the end of round $t$. In the \emph{bandit setting} (Section~\ref{section:omg_bandit}), players can only observe the entry of $A_t$ indexed by $(i_t,j_t)$ at the end of round $t$, where $i_t$ and $j_t$ are the actions sampled from the probability distributions associated with their mixed strategies $(x_t, y_t)$. 



\subsection{Main Contributions}
\label{sec:contributions}
In addition to introducing a novel problem setting, the main contributions of the present work are as follows. 
\vspace*{-10pt}
\begin{itemize}
\setlength{\itemsep}{-5pt}
    \item First, we show that a natural ``na\"ive'' approach, where each player simply aims to minimize their individual regret, will fail to produce a sublinear NE regret algorithm, in the sense of \eqref{eq:NE-regret}, regardless of the players' no-regret strategies (Theorem~\ref{thm:choose_one_omg}). 
    \item Second, in the full information setting, we provide an algorithm for the OMG problem that achieves a NE regret of $O(\max\{\ln(d_1),\ln(d_2)\}\ln(T)\sqrt{T})$ (Theorem~\ref{thm:omg_rftl_regret}). Note that the regret depends logarithmically on the number of actions, allowing us to handle scenarios where the players have exponentially many actions available.
    \item  Third, we propose an algorithm for the bandit setting that achieves an NE regret of order $O((\max\{d_1,d_2\})^{5/3}T^{5/6})$ (Theorem~\ref{no_bandit_regret}).
    \item Fourth, we show empirically how our algorithm can be used to prevent mode collapse when training GANs in a basic setup (Section \ref{sec:GANs}).
\end{itemize}


\subsection{Related Work}
\label{sec:liter-rev}

The reader familiar with Online Convex Optimization (OCO) may find it closely related to the OMG problem. In the OCO setting, a player is given a convex, closed, and bounded action set $X$, and must repeatedly choose an action $x_t\in X$ before the convex function $f_t(x):X\rightarrow \mathbb{R}$ is revealed. The player's goal is to obtain sublinear \textit{individual regret} defined as $\sum_{t=1}^T f_t(x_t) - \min_{x\in X}\sum_{t=1}^T f_t(x)$. This problem is well studied and several algorithms such as Online Gradient Descent \cite{zinkevich2003online}, Regularized Follow the Leader \cite{shalev2007online,abernethy2009competing} and Perturbed Follow the Leader \cite{kalai2002} achieve optimal individual regret bounds that scale as $O(\sqrt{T})$. 
The most natural (although incorrect) approach to attack the OMG problem is to equip each of the players with a sublinear individual regret algorithm. However, we will show in Section~\ref{section:choose_one_omg} that if both players use an algorithm that guarantees sublinear individual regret, then it is impossible to achieve sublinear NE regret when the payoff matrices are chosen adversarially. In other words, the algorithms for the OCO setting cannot be directly applied to the OMG problem considered in this paper.

%

We now discuss some related works that focus on learning in games. \cite{singh2000nash} study a two player, two-action general sum static game. They show that if both players use Infinitesimal Gradient Ascent, either the strategy pair will converge to a Nash Equilibrium (NE), or even if they do not, then the average payoffs are close to that of the NE. A result of similar flavor was derived in \cite{cesa2007improved} for any zero-sum convex-concave game. Given a payoff function $\EL(x,y)$, they show that if both players minimize their individual-regrets, then the average of actions $(\bar{x},\bar{y})$ will satisfy $|\EL(\bar{x},\bar{y})-\EL(x^*,y^*)|\to 0$ as $T \to \infty$, where $(x^*,y^*)$ is a NE.  \cite{bowling2001convergence} improve upon the result of \cite{singh2000nash} by proposing an algorithm called WoLF (Win or Learn Fast), which is a modification of gradient ascent; they show that the iterates of their algorithm indeed converge to a NE. \cite{conitzer2007awesome} further improve the results in \cite{singh2000nash} and \cite{bowling2005convergence} 
by developing an algorithm called GIGA-WoLF for multi-player nonzero sum static games. Their algorithm learns to play optimally against stationary opponents; when used in self-play, the actions chosen by the algorithm converge to a NE. 
More recently, \cite{balduzzi2018mechanics} studied general multi-player static games and show that by decomposing and classifying the second order dynamics of these games, one can prevent cycling behavior to find NE. 
We note that unlike our paper, all of the papers above consider repeated games with a static payoff matrix, whereas we allow the payoff matrix to change arbitrarily.
An exception is the work by \cite{ho2016role}, who consider the same setting as our OMG problem; however their paper only shows that the sum of the individual regrets of both players is sublinear and does not study convergence to NE.

Related to the OMG problem with bandit feedback is the seminal work of \cite{flaxman2005online}.
They provide the first sublinear regret bound for Online Convex Optimization with bandit feedback,
using a one-point estimate of the gradient. The one-point gradient estimate used in \cite{flaxman2005online} is similar to those independently proposed in \cite{granichin1989stochastic} and in \cite{spall1997one}. The regret bound provided in \cite{flaxman2005online} is $O(T^{3/4})$, which is suboptimal. In \cite{abernethy2009competing}, the authors give the first $O(\sqrt{T})$ bound for the special case when the functions are linear. More recently, \cite{hazan2016optimal} and \cite{bubeck2016kernel} designed the first efficient algorithms with $\tilde{O}(poly(d)\sqrt{T})$ regret  for the general online convex optimization case; unfortunately, the dependence on the dimension $d$ in the regret rate is a very large polynomial. Our one-point matrix estimate is most closely related to the random estimator in \cite{auer1995gambling} for linear functions. It is possible to use the more sophisticated techniques from \cite{abernethy2009competing,hazan2016optimal,bubeck2016kernel} to improve our NE regret bound in section \ref{section:omg_bandit}; however, the result does not seem to be immediate and we leave this as future work. 

\section{Preliminaries}
In this section we introduce notation and definitions that will be used throughout the paper.

\subsection{Notation}
By default, all vectors are column vectors. A vector with entries $x_1,...,x_d$ is written as $x = [x_1;...;x_d] = [x_1,...,x_d]^\top$, where $\top$ denotes the transpose. For a matrix $A$, let $A_{ij}$ be the entry in the $i$-th row and $j$-th column. 


\subsection{Convex Functions}
For any $H>0$ we say that a function $f:X\rightarrow \mathbb{R}$ is $H$-strongly convex with respect to a norm $\Vert \cdot \Vert$, if for any $x_1, x_2 \in X$,  it holds that
\begin{align*}
f(x_1) \geq f(x_2) + \nabla f(x_2)^\top(x_1-x_2) + \frac{H}{2}\Vert x_1-x_2 \Vert^2.
\end{align*}
Here, $\nabla f(x)$ denotes any subgradient of $f$ at $x$. Strong convexity implies that the optimization problem $\min_{x \in X} f(x)$ has
a unique solution. If $H=0$ we simply say that the function is convex. We say a function $g$ is $H$-strongly concave if $-g$ is $H$-strongly convex. Furthermore, we say a function $\EL(x,y)$ is $H$-strongly convex-concave if for any fixed $y_0 \in Y$, the function $\EL(x,y_0)$ is $H$-strongly convex in $x$, and for any fixed $x_0\in X$, the function $\EL(x_0,y)$ is $H$-strongly concave in $y$.

\subsection{Saddle Points and Nash Equilibra}
A pair $(x^*,y^*)$ is called a saddle point for $\EL:X\times Y \rightarrow \mathbb{R}$ if for any $x\in X$ and $y \in Y$, we have
\begin{equation} \label{def_sp}
\EL(x^*,y) \leq \EL(x^*,y^*) \leq \EL(x,y^*).
\end{equation}
It is well known that if $\EL$ is convex-concave, and $X$ and $Y$ are convex and compact sets, there always exists at least one saddle point \citep[see e.g.][]{boyd2004convex}. Moreover, if $\EL$ is strongly convex-concave, the saddle point is unique.


A saddle point is also known as a Nash equilibrium for two-player zero-sum games \cite{nash1951non}. 
In a matrix game, the payoff function $\EL(x,y) = x^\top A y$ is bilinear, and therefore is convex-concave. The action spaces of the two players are $X=\Delta_{d_1}$ and $Y=\Delta_{d_2}$, which are convex and compact. As a result, there always exists a Nash equilibrium for any matrix game. The famous von Neumann minimax theorem states that $\min_{x\in \Delta_{d_1}} \max_{y\in \Delta_{d_2}} x^\top A y=  \max_{y\in \Delta_{d_2}} \min_{x\in \Delta_{d_1}}  x^\top A y$. If Player 1 chooses $x^* \in \arg \min_{x\in \Delta_{d_1}} \max_{y\in \Delta_{d_2}} x^\top A y$ and Player 2 chooses $y^* \in \arg \max_{y\in \Delta_{d_2}} \min_{x\in \Delta_{d_1}}  x^\top A y$, the pair $(x^*,y^*)$ is an equilibrium of the game \cite{morgenstern1953theory}. 

\subsection{Lipschitz Continuity}
We say a function $f:X\rightarrow \mathbb{R}$ is $G$-Lipschitz continuous with respect to a norm $\Vert \cdot \Vert$ if for all $x,y \in X$ it holds that
\begin{align*}
|f(x) - f(y)| \leq G \Vert x - y\Vert
\end{align*}
It is well known that the previous inequality holds if and only if
\begin{align*}
\Vert \nabla f(x) \Vert_* \leq G 
\end{align*}
for all $x\in X$, where $\Vert \cdot \Vert_*$ denotes the dual norm of $\Vert \cdot \Vert$ \cite{boyd2004convex,shalev2012online}.
Similarly, we say a function $\EL(x,y)$ is $G$-Lipschitz continuous with respect to a norm $\Vert \cdot \Vert$ if 
\begin{align*}
|\EL(x_1, y_1)- \EL(x_2, y_2)| \leq G \Vert [x_1;y_1] - [x_2;y_2] \Vert.
\end{align*}
for any $x_1, x_2 \in X$ and any $y_1, y_2 \in Y$. Again, the previous inequality holds if and only if
\begin{align*}
\Vert [\nabla_x \EL(x,y); \nabla_y \EL(x,y)] \Vert_* \leq G 
\end{align*}
for all $x\in X$, $y \in Y$.

\begin{lemma}\label{bilinear_lipschitz}
Consider a matrix $A$. If the absolute value of each entry of $A$ is bounded by $c>0$, then the function $\EL(x,y) = x^{\top}A y$ is $ G_{\EL}^{\Vert \cdot \Vert_2}$-Lipschitz continuous with respect to $\Vert \cdot \Vert_2$, where $ G_{\EL}^{\Vert \cdot \Vert_2}  = \sqrt{c}\left(\sqrt{d_1}+\sqrt{d_2}\right)$. The function $\EL$ is also $G_{\EL}^{\Vert \cdot \Vert_1}$-Lipschitz continuous with respect to norm $\Vert \cdot \Vert_1$, where $G_{\EL}^{\Vert \cdot \Vert_1} = c$.
\end{lemma}

 
 \section{Challenges of the OMG Problem: An Impossibility Result}\label{section:choose_one_omg}
 
Recall that we defined the Online Matrix Games (OMG) problem in Section~\ref{sec:problem-formulation}, where two players play a zero-sum games for $T$ rounds. The sequence of payoff matrices $\{A_t\}^T_{t=1}$ is selected arbitrarily. In each round $t\in[T]$, both players choose their strategies before the payoff matrix $A_t$ is revealed. The goal is to find strategies under which the players' average payoffs are close to the Nash Equilibrium of the game with payoff matrix $\sum_{t=1}^T A_t$.
 
Perhaps the most natural (albeit futile) approach to attack the OMG problem is to equip each of the players with a sublinear individual regret algorithm to generate a sequence of iterates $\{x_t,y_t\}_{t=1}^T$. We gave a few examples of Online Convex Optimization (OCO) algorithms that guarantee $O(\sqrt{T})$ regret in Section~\ref{sec:liter-rev}.
However, if each player minimizes its individual regret greedily using OCO, 
this approach only implies that
$
\sum_{t=1}^T x_t^\top A_t y_t - \min_{x\in \Delta_X} \sum_{t=1}^T x A_t y_t=O(\sqrt{T})$, and  
$
\max_{y\in \Delta_Y} \sum_{t=1}^T x_t^\top A_t y - \sum_{t=1}^T x_t^\top A_t y_t = O(\sqrt{T}).
$
Notice that the quantity $\MM \sum_{t=1}^T x^\top A_t y$ associated with the Nash Equilibrium in equation ~\eqref{eq:NE-regret} does not even appear in these bounds. The reader familiar with saddle point computation may wonder how the so-called `duality gap' \cite{bubeck2015convex}: $\max_{y\in \Delta_Y} \sum_{t=1}^T x_t^\top A_t y - \min_{x\in \Delta_X} \sum_{t=1}^T x A_t y_t=O(\sqrt{T})$ relates to achieving sublinear NE regret. It is easy to see that the duality gap is the sum of individual regret of both players. In view of Theorem \ref{thm:choose_one_omg} we will see that NE regret and the duality gap are in some sense incompatible.

In this section we present a result that shows that there is no algorithm that \textit{simultaneously} achieves sublinear NE regret and individual regret for both players. This implies that if both players individually use any existing algorithm from OCO they would inevitably fail to solve the OMG problem.
\begin{theorem}\label{thm:choose_one_omg}
Consider any algorithm that selects a sequence of $x_t, y_t$ pairs given the past payoff matrices $A_1, \ldots, A_{t-1}$. Consider the following three objectives:
\begin{eqnarray}
 \label{eq:xyregret} \left\vert \sum_{t=1}^Tx_t^\top A_t y_t - \MMD  \sum_{t=1}^Tx^\top A_t y \right \vert & = & o(T), \\
  \label{eq:xregret} \sum_{t=1}^Tx_t^\top A_t y_t - \min_{x\in \Delta_X} \sum_{t=1}^T x^\top A_t y_t  & = & o(T), \\
  \label{eq:yregret} \max_{y\in \Delta_Y} \sum_{t=1}^T x_t^\top A_t y - \sum_{t=1}^Tx_t^\top A_t y_t  & = & o(T).
\end{eqnarray}
Then there exists an (adversarially-chosen) sequence $A_1, A_2, \ldots$ such that not all of \eqref{eq:xyregret}, \eqref{eq:xregret}, and \eqref{eq:yregret}, are true.
\end{theorem} 
A full proof of the result is shown in the Appendix, but here we give a sketch. The main idea is to construct two parallel scenarios, each with their own sequences of payoff matrices. The two scenarios will be identical for the first $T/2$ periods but are different for the rest of the horizon. In our particular construction, in both scenarios the players play the well known ``matching-pennies'' game for the first $T/2$ periods, then in first scenario they play a game with equal payoffs for all of their actions and in the second scenario they play a game where Player 1 is indifferent between its actions. One can show that if all three quantities in the statement of the theorem are $o(T)$ in the first scenario, then we prove that at least one of them is $\Omega(T)$ in the second one which yields the result. This suggests that the machinery for OCO, which minimizes individual regret, cannot be directly applied to the OMG problem.

\section{Online Matrix Games: Full Information}
\label{sec:full-info}

\subsection{Saddle Point Regularized Follow-the-Leader}

In this section we propose an algorithm to solve the OMG problem in the full information setting.
In fact, we will consider the algorithm in a slighly more general setting than the OMG problem, allowing the sequence of payoff functions to be specified by arbitrary convex-concave Lipschitz functions, and the action sets of Player 1 and Player 2 ($X\subset \mathbb{R}^n$ and $Y\subset \mathbb{R}^n$ respectively) to be arbitrary convex compact sets.

Let the sequence of convex-concave functions be $\{\bar{\EL}_t(x,y)\}_{t=1}^T$, which are $G_{\bar{\EL}}$-Lipschitz with respect to some norm $\Vert \cdot \Vert$. 
We propose an algorithm called Saddle Point Regularized Follow the Leader (\textsc{SP-RFTL}), shown in Algorithm~\ref{alg:SPRFTL}.

\begin{algorithm}[tbh]
\caption{Saddle-Point Regularized-Follow-the-Leader (\textsc{SP-RFTL})}
\label{alg:SPRFTL}
\begin{algorithmic}
  \STATE {\bfseries input:} $x_1 \in X$, $y_1 \in Y$, parameters: $\eta>0$, strongly convex functions $R_X$, $R_Y$
  \FOR{$t=1,...T$}
  \STATE Play $(x_t,y_t)$ 
  \STATE Observe $\bar{\EL}_t$
  \STATE $\EL_t(x,y) \gets \bar{\EL}_t + \frac{1}{\eta}R_X(x)-\frac{1}{\eta}R_Y(y)$
  \STATE $x_{t+1}\leftarrow \arg \min_{x\in X} \max_{y \in Y} \sum_{\tau=1}^t \EL_t(x,y)$
  \STATE $y_{t+1}\leftarrow \arg \max_{y \in Y} \min_{x\in X}  \sum_{\tau=1}^t \EL_t(x,y) $
  \ENDFOR
\end{algorithmic}
\end{algorithm}

The regularizers $R_X, R_Y$ are used as input for the algorithm. We will choose regularizers that are strongly convex with respect to norm $ \Vert \cdot \Vert $, and $G_{R_1}$ and $G_{R_2}$ Lipschitz with respect to norm $ \Vert \cdot \Vert $, which means that $\Vert \nabla R_X(x)\Vert_* \leq G_{R_1}$ for all $x\in X$, and $\Vert \nabla R_Y(y)\Vert_* \leq G_{R_2}$ for, all $y\in Y$. Finally, we assume $R_X(x)\geq 0 $ for all $x\in X$ and $R_Y(y)\geq 0 $ for all $y\in Y$.

The main difference between \textsc{SP-RFTL} and the well known Regularized Follow the Leader (\textsc{RFTL}) algorithm \cite{shalev2007online,abernethy2009competing} is that in \textsc{SP-RFTL} both players update jointly and play the saddle point of the sum of regularized games observed so far. In particular, they disregard their previous actions. In contrast, the updates for \textsc{RFTL} would be 
\begin{align*}
    x^{RFTL}_{t+1}\leftarrow \arg \min_{x\in X} \sum_{\tau=1}^t \left[\bar{\EL}_\tau(x, y^{RFTL}_\tau)+\frac{1}{\eta}R_X(x)\right]\\
    y^{RFTL}_{t+1}\leftarrow \arg \max_{y\in Y}  \sum_{\tau=1}^t \left[\bar{\EL}_\tau(x^{RFTL}_\tau, y)-\frac{1}{\eta}R_Y(y)\right]
\end{align*} for $t=2,...,T$, and $x^{RFTL}_1$, $y^{RFTL}_1$ are chosen as to minimize $R_X(x)$ and $-R_Y(y)$ in their respective sets $X,Y$. It is easy to see that the sequence of iterates is in general not the same. In fact, in view of Theorem~\ref{thm:choose_one_omg} we know that \textsc{RFTL} can not achieve sublinear NE regret when the sequence of functions is chosen arbitrarily. One last remark about the algorithm is that as $T\rightarrow \infty$ the last iterates $(x_{T+1},y_{T+1})$ will converge to the set of NE of the average game $\frac{1}{T} \sum_{t=1}^T \bar{\EL}_t$. To see this, observe that if $\eta=\sqrt{T}$ then $x_{T+1} \leftarrow \arg \min_{x\in X} \max_{y \in Y} \frac{1}{T} \sum_{t=1}^T \left[ \bar{\EL}_t(x,y) \right]+ \frac{1}{\sqrt{T}}R_X(x) - \frac{1}{\sqrt{T}}R_Y(y)$ i.e. $x_{T+1}$ solves the average problem where the regularization is vanishing, and a similar expression can be written for $y_{T+1}$. This is in contrast with many of the results mentioned in Section \ref{sec:liter-rev} where it is the \textit{average} of the iterates which is an approximate equilibrium.

We have the following guarantee for \textsc{SP-RFTL}.

\begin{theorem}\label{theorem:sp_regret_convex_concave}
For $t=1,...,T$, let $\bar{\EL}_{t}$ be $G_{\bar{\EL}}$-Lipschitz with respect to norm $\Vert \cdot \Vert$. Let $R_X$, $R_Y$ be strongly convex functions with respect to the same norm, let $G_{R_X},G_{R_Y}$ be the Lipschitz constants of $R_X$, $R_Y$ with respect to the same norm. Let $\{(x_t,y_t)\}_{t=1}^T$ be the iterates generated by \textsc{SP-RFTL} when run on convex-concave functions $\{\bar{\EL}_{t}(x,y)\}_{t=1}^T$. It holds that 
\begin{align*}
& \left| \sum_{t=1}^T \bar{\EL}_t(x_t,y_t) - \MM \sum_{t=1}^T \bar{\EL}_t(x,y) \right|  \\
\leq & 8 \eta \left[G_{\bar{\EL}}+ \frac{1}{\eta}\max(G_{R_X}, G_{R_Y})\right]^2 ( 1 + \ln(T) )\\
&+ \frac{T}{\eta} \max_{y\in Y} R_Y(y) +  \frac{T}{\eta} \max_{x\in X} R_X(x) \; = \;   O\left( \sqrt{T \ln(T)}\right),
\end{align*}
where the last equality follows by choosing $\eta=\frac{\sqrt{T}}{\ln(T)}$.
\end{theorem}
A formal proof of the theorem is provided in the Appendix and a sketch will be given shortly.

We note that the bound in Theorem \ref{theorem:sp_regret_convex_concave} holds for general convex-concave functions, however the dependence on the dimension is hidden on the Lipschitz constants and the choice of regularizer. It is easy to check that if one chooses $\Vert \cdot \Vert_2^2$ as regularizer, and the functions $\{\EL_t\}_{t=1}^T$ are $G$-Lipschitz continuous with respect to norm $\Vert \cdot \Vert_2^2$, then the NE regret bound will be $O(n\ln(T)\sqrt{T})$. 

We now provide a sketch of the proof of Theorem \ref{theorem:sp_regret_convex_concave}.
Define $\EL_t(x,y) \triangleq \bar{\EL}_t(x,y) + \frac{1}{\eta}R_X(x) - \frac{1}{\eta}R_Y(y)$. Notice that it is $\frac{1}{\eta}$-strongly convex in $x$ with respect to norm $\Vert \cdot \Vert$ for all $y\in Y$ and $\frac{1}{\eta}$-strongly concave with respect to norm $\Vert \cdot \Vert$ for all $x \in X$. Additionally, notice that $\EL_t$ is $G_\EL\triangleq G_{\bar{\EL}} + \frac{1}{\eta} (G_{R_X}+G_{R_Y})$-Lipschitz with respect to norm $ \Vert \cdot \Vert $. Finally, notice that for $t=1,...,T$, all $x \in X$ and all $y\in Y$ it holds that
\begin{equation}\label{eq:diff_bar_not_bar}
-\frac{1}{\eta}R_Y(y) \leq \EL_t(x,y) - \bar{\EL}_t(x,y) \leq \frac{1}{\eta} R_X(x)
\end{equation}

The following lemma shows that the value of the convex-concave games defined by $\sum_{t=1}^T \EL_t$ and $\sum_{t=1}^T \bar{\EL}_t$ are not too far from each other.
\begin{lemma}\label{lemma:mm_bar_not_bar}
Let 
\begin{align*}
    \textstyle \bar{x}_{T+1} \in \arg \min_{x\in X} \max_{y\in Y} \sum_{t=1}^T\bar{\EL}_t(x,y), \\
    \textstyle \bar{y}_{T+1} \in \arg \max_{y\in Y} \min_{x\in X}  \sum_{t=1}^T\bar{\EL}_t(x,y).
\end{align*} It holds that 
\begin{align*}
& -\frac{T}{\eta} R_Y(\bar{y}_{T+1}) \\
& \leq \MM \sum_{t=1}^T \EL_t (x,y) - \MM \sum_{t=1}^T \bar{\EL}_t(x,y) \\
& \leq \frac{T}{\eta} R_X(\bar{x}_{T+1}).
\end{align*}
\end{lemma}

To prove the NE regret bound, we note that \textsc{SP-RFTL} is running a Follow-the-Leader scheme on functions $\{\EL_{t=1}^T\}$ \cite{kalai2002}. With the next two lemmas one can show that the NE regret of the players relative to functions $\{\EL\}_{t=1}^T$ is small.

\begin{lemma}\label{lemma:loss_BTL}Let $\{(x_t,y_t)\}_{t=1}^T$ be the iterates of \textsc{SP-RFTL}. It holds that
\begin{align*}
&\textstyle -G_\EL \sum_{t=1}^T \Vert x_t-x_{t+1} \Vert  \\
& \leq \sum_{t=1}^T \EL_t(x_{t+1},y_{t+1}) - \MM \sum_{t=1}^T \EL_t(x,y)  \\
& \textstyle \leq G_\EL \sum_{t=1}^T \Vert y_t-y_{t+1} \Vert.
\end{align*}
\end{lemma}

\begin{lemma} Let $\{(x_t,y_t)\}_{t=1}^T$ be the sequence of iterates generated by the algorithm. It holds that
\begin{align*}
&\Vert x_{t}-x_{t+1}\Vert + \Vert y_t - y_{t+1}\Vert \\ 
& \leq \frac{4 \eta }{t} \left[G_{\bar{\EL}}+ \frac{1}{\eta}\max(G_{R_X}, G_{R_Y})\right].
\end{align*}
\end{lemma}

Combining the NE regret bound obtained on functions $\{\EL\}_{t=1}^T$ together with Lemma \ref{lemma:mm_bar_not_bar} and equation \eqref{eq:diff_bar_not_bar} yields the theorem.

\subsection{Logarithmic Dependence on the Dimension of the Action Spaces}


Previously, we analyzed the OMG problem by treating the payoff functions as general convex-concave functions and the action spaces as general convex compact sets. We explained that in general one should expect to achieve NE regret which depends linearly in the dimension of the problem.
The goal in this section is to obtain sharper NE regret bounds that scale as $O(\ln(T)\sqrt{T} \ln(\max(d_1,d_2)))$ by exploiting the geometry of the decision sets $\Delta_X, \Delta_Y$ and the bilinear structure of the payoff functions. 
This allows us to solve games which may have exponentially many actions, which often arise in combinatorial optimization settings.


The plan to obtain the desired NE regret bounds in this more restrictive setting is to use the negative entropy as a regularization function (which is strongly convex with respect to $\Vert \cdot \Vert_1$), that is $R_X(x) = \sum_{i=1}^{d_1} x_i\ln(x_i) + \ln(d_1)$ and $R_Y(y)= \sum_{i=1}^{d_2} y_i \ln(y_i)+\ln(d_2)$ where the extra logarithmic terms ensure $R_X, R_Y$ are nonnegative everywhere in their respective simplexes. Unfortunately, the negative entropy is not Lipschitz over the simplex, so we can not leverage our result from Theorem~\ref{theorem:sp_regret_convex_concave}. To deal with this challenge, we will restrict the new algorithm to play over a restricted simplex:\footnote{We will also use the notation $\Delta_{X,\theta}$ and $\Delta_{Y,\theta}$ to mean the restricted simplex of Player 1 and 2, respectively}
\begin{equation}
\Delta_\theta = \{z \in \mathbb{R}^d: \Vert z\Vert_1=1, z_i\geq \theta, i=1,...,d \}.
\end{equation}
The tuning parameter $\theta \in [0,1/d]$ used for the algorithm will be defined later in the analysis. (Notice that when $\theta > {1}/{d}$, the set is empty.) We have the following result.
\begin{lemma}\label{lemma:entropy_lipschitz}
The function $R(x)\triangleq \sum_{i=1}^d x_i \ln(x_i)$ is $G_R$-Lipschitz continuous with respect to $\Vert \cdot \Vert_1$ over $\Delta_\theta$ with $G_R = \max\{|\ln(\theta)|,1\}$.
\end{lemma}

The  algorithm \textsc{Online-Matrix-Games Regularized-Follow-the-Leader} is an instantiation of \textsc{SP-RFTL} with a particular choice of regularization functions, which are nonegative and Lipschitz over the sets $\Delta_{X,\theta}$, $\Delta_{Y,\theta}$. With this, we can 
prove a NE regret bound for the OMG problem. For the remainder of the paper, the regularization functions will be set as follows:
\begin{eqnarray*}
R_X(x) & \triangleq & \textstyle \sum_{i=1}^{d_1} x_i \ln(x_i)+\ln(d_1),\\
R_Y(y) & \triangleq & \textstyle \sum_{i=1}^{d_2} y_i\ln(y_i)+\ln(d_2).
\end{eqnarray*}
\begin{algorithm}[tbh]
\caption{Online-Matrix-Games Regularized-Follow-the-Regularized-Leader (\textsc{OMG-RFTL})}
\label{alg:OMG-RFTL}
\begin{algorithmic}
  \STATE {\bfseries input:} $x_1 \in \Delta_{X,\theta}\subset \mathbb{R}^{d_1}$, $y_1 \in \Delta_{Y,\theta}\subset \mathbb{R}^{d_2}$, parameters: $\eta>0$, $\theta<\min \{\frac{1}{d_1}, \frac{1}{d_2}\}$.
  \FOR{$t=1,...T$}
  \STATE Play $(x_t,y_t)$, observe matrix $A_t$
  \STATE $\bar{\EL}_t \gets x^\top A_t y$
  \STATE $\EL_t(x,y) \gets \bar{\EL}_t + \frac{1}{\eta}R_X(x)-\frac{1}{\eta}R_Y(y)$
  \STATE $x_{t+1}\leftarrow \arg \min_{x\in \Delta_{X,\theta}} \max_{y \in \Delta_{Y, \theta}} \sum_{\tau=1}^t \EL_t(x,y)$
  \STATE $y_{t+1}\leftarrow \arg \max_{y \in \Delta_{Y, \theta}} \min_{x\in \Delta_{X,\theta}}  \sum_{\tau=1}^t \EL_t(x,y) $
  \ENDFOR
\end{algorithmic}
\end{algorithm}

We have the following guarantee for \textsc{OMG-RFTL}. 
\begin{theorem}\label{thm:omg_rftl_regret}
Let $\{A_t\}_{t=1}^T$ be an arbitrary sequence of matrices with entries bounded between $[-1, 1]$. Let $G_{\bar{\EL}}$ be the Lipschitz constant (with respect to $\Vert \cdot \Vert_1$) of $\bar{\EL}_t \triangleq x^\top A_t y$ for $t=1,...,T$. Let $\{(x_t,y_t)\}_{t=1}^T$ be the iterates of \textsc{OMG-RFTL}) and choose $\theta = e^{-\eta G_{\bar{\EL}}}\leq \min\{\frac{1}{d_1}, \frac{1}{d_2}\}$ such that $\frac{|\ln(\theta)|}{\eta} = G_{\bar{\EL}}$. Set $\eta = \frac{\sqrt{T}}{G_{\bar{\EL}}}$.  It holds that
\begin{align*}
& \left|\sum_{t=1}^T x_t^\top A_t y_t - \MMD \sum_{t=1}^T x^\top A_t y \right|\\
& \leq 32 G_{\bar{\EL}} \sqrt{T} (1 + \ln(T)) + 2 \sqrt{T} \max\{\ln d_1 ,\ln d_2\} +\\
&\quad 2 \max\{d_1,d_2\} G_{\bar{\EL}} T e^{-\sqrt{T}}\\
& = O\left(\ln(T)\sqrt{T} +  \sqrt{T} \max\{\ln d_1 ,\ln d_2\}\right) +\\
&\quad o(1)\max\{d_1,d_2\}.
\end{align*}
\end{theorem}

A full proof of the theorem can be found in the Appendix. We now give a sketch of the proof. Since the algorithm selects actions over the restricted simplex, we must quantify the potential loss in the NE regret bound imposed by this restriction. The next two lemmas make this precise.
\begin{lemma}\label{lemma:dist_proj_sp}
Let $z^* \in \Delta \subset \mathbb{R}^d$ define $z^*_p \triangleq \arg \min_{z\in \Delta_\theta} \Vert z - z^*\Vert_1$, with $\theta \leq \frac{1}{d}$. Notice $z^*_p$ is unique since it is a projection. It holds that $\Vert z^*_p - z^*\Vert_1 \leq 2\theta(d-1)$.
\end{lemma}

\begin{lemma}\label{lemma:sp_val_error_theta}
Let $\{\bar{\EL}_t(x,y)\}_{t=1}^T$ be an arbitrary sequence of convex-concave functions, $\bar{\EL}_t:\Delta_X \times \Delta_Y \rightarrow \mathbb{R}$, that are $G_{\bar{\EL}}$-Lipschitz with respect to $\Vert \cdot \Vert_1$. With $\Delta_X \subseteq \mathbb{R}^{d_1}$, and $\Delta_Y \subseteq \mathbb{R}^{d_2}$. It holds that 
\begin{align*}
& - G_{\bar{\EL}} T \Vert x^*_p - x^*\Vert_1 \\ 
\leq  &\; \MMD \sum_{t=1}^T \bar{\EL}_t(x,y) - \MMDt \sum_{t=1}^T \bar{\EL}_t(x,y) \\
\leq & \; G_{\bar{\EL}}T \Vert y^*_p - y^*\Vert_1.
\end{align*}
\end{lemma}
Combining the previous two lemmas and Theorem \ref{theorem:sp_regret_convex_concave}, one can show the NE regret bound for \textsc{OMG-RFTL} holds.

\section{Online Matrix Games: Bandit Feedback}\label{section:omg_bandit}

In this section we focus on the OMG problem under bandit feedback. In this setting, the players observe in every round only the payoff corresponding to the chosen actions. If Player 1 chooses action $i$, Player 2 chooses action $j$, and the payoff matrix at that time step is $A_t$, then the players observe only $(A_t)_{ij}$ instead of the full matrix $A_t$. The limited feedback makes the problem significantly more challenging than the full information one: the players must find a way to \textit{exploit} (use all previous information to try to play a Nash Equilibrium) and \textit{explore} (try to estimate $A_t$ in every round). This problem resembles that of Online Bandit Optimization \cite{flaxman2005online,auer1995gambling,bubeck2016kernel,hazan2016optimal}, while the main difference is that with one function evaluation we must estimate a matrix $A_t$ instead of the gradients $\nabla_x \EL_t(x,y)$ and $\nabla_y \EL_t(x,y)$ where $\EL_t = x^\top A_t y$.

Before proceeding we establish some useful notation. 
For $i=1,...,d$, let $e_i \in \mathbb{R}^d$  be the collection of standard unit vectors i.e. $e_i$ is the vector that has a $1$ in the $i$-th entry and $0$ in the rest. Let $e_{x,t}$ be the standard unit vector corresponding to the decision made by Player 1 for round $t$, define $e_{y,t}$ similarly. Notice that under bandit feedback, in round $t$ both players only observe the quantity $e_{x,t}^{\top} A_t e_{y,t}$.

\subsection{A One-Point Estimate for $\EL(x,y) = x^\top A y$}
As explained previously, in each round $t$ the players must estimate $A_t$ by observing only one of its entries. To this end, we allow the players to share with each other their decisions and to randomize \textit{jointly} (a similar assumption is used to define correlated equilibria in zero-sum games, see \cite{aumann1987correlated}). The following result shows how to build a random estimate of $A$ by observing only one of its entries.
 
\begin{theorem}\label{thm:hess_estimate}
Let $x\in \Delta_{X,\delta}, y\in \Delta_{Y,\delta}$ with $d_1,d_2\geq2$ and $\delta >0$. Sample $i' \sim x, j'\sim y$. Let $\hat{A}$ be the $d_1\times d_2$ matrix with $\hat{A}_{i,j}=0$ for all $i,j$ such that $i\neq i'$ and $j\neq j'$ and $\hat{A}_{i',j'} = \frac{A_{i',j'}}{x(i')y(j')}$. It holds that 
\begin{equation*} 
\mathbb{E}_{i' \sim x, j' \sim y} [ \hat{A}] = A.
\end{equation*}
\end{theorem}

\subsection{Bandit Online Matrix Games \textsc{RFTL}}
We now present an algorithm that ensures sublinear (i.e. $o(T)$) NE regret under bandit feedback for the OMG problem that holds against an adaptive adversary. By adaptive adversary, we mean that the payoff matrices $A_t$ can depend on the players' actions up to time $t-1$; in particular, we assume the adversary does not observe the actions chosen by the players for time period $t$ when choosing $A_t$. We consider an algorithm that runs \textsc{OMG-RFTL} on a sequence of functions $\hat{\EL}_t \triangleq x^\top \hat{A}_t y$, where $\hat{A}_t$ is the unbiased one-point estimate of $A_t$ derived in Theorem~\ref{thm:hess_estimate}. Recall that the iterates of \textsc{OMG-RFTL} algorithm are distributions over the possible actions of both players. In order to generate the estimate $\hat{A}_t$, both players will sample an action from their distributions and weigh their observation with the inverse probability of obtaining that observation. 

\begin{algorithm}[tbh]
\caption{Bandit Online-Matrix-Games Regularized-Follow-the-Leader (\textsc{Bandit-OMG-RFTL})}
\label{alg: BOMGFTRL}
\begin{algorithmic}
  \STATE {\bfseries input:} $x_1 \in \Delta_{X,\delta}\subset \mathbb{R}^{d_1}$, $y_1 \in \Delta_{Y,\delta}\subset \mathbb{R}^{d_2}$, parameters: $\eta>0$, $0<\delta<\min \{\frac{1}{d_1}, \frac{1}{d_2}\}$.
  \FOR{$t=1,...T$}
  \STATE Sample independently $e_{x,t} \sim \tilde{x}_t$ and $e_{y,t} \sim \tilde{y}_t$
  \STATE Observe $e_{x,t}^{\top} A_t e_{y,t}$
  \STATE  Build $\hat{A_t}$ as in Theorem \ref{thm:hess_estimate} using  $e_{x,t}^{\top} A_t e_{y,t}, x_t, y_t$
    \STATE $\hat{\EL}_t \gets x^\top \hat{A}_t y$
    \STATE $\EL_t(x,y) \gets \hat{\EL}_t + \frac{1}{\eta}R_X(x)-\frac{1}{\eta}R_Y(y)$
  \STATE $x_{t+1}\leftarrow \arg \min_{x\in \Delta_{X,\theta}} \max_{y \in \Delta_{Y, \theta}} \sum_{\tau=1}^t \EL_t(x,y)$
  \STATE $y_{t+1}\leftarrow \arg \max_{y \in \Delta_{Y, \theta}} \min_{x\in \Delta_{X,\theta}}  \sum_{\tau=1}^t \EL_t(x,y) $
  \ENDFOR
\end{algorithmic}
\end{algorithm}

We have the following guarantee for \textsc{Bandit-OMG-RFTL}. 
\begin{theorem}\label{no_bandit_regret}
Let $\{A_t\}_{t=1}^T$ be any sequence of payoff matrices chosen by an adaptive adversary. Let $\{e_{x,t},e_{y,t}\}_{t=1}^T$ be the iterates generated by \textsc{Bandit-OMG-FTRL}. Setting $\delta = \frac{1}{T^{1/6}}$, $\eta = T^{1/6}$ ensures
\begin{align*} 
&\left| \mathbb{E}  \left[\sum_{t=1}^T e_{x,t}^{\top}A_t e_{y,t} - \MM  \sum_{t=1}^T x^{\top}A_t y \right] \right|  \\ \leq& O((d_1 + d_2) \ln(T) T^{5/6})
 \end{align*}
where the expectation is taken with respect to randomization in the algorithm.
\end{theorem}

We now give a sketch of the proof. 
The total payoff given to each of the players is given by $\sum_{t=1}^T e_{x,t}^\top A_t e_{y,t}$ so we must relate this quantity to the iterates $\{x_t,y_t\}_{t=1}^T$ of \textsc{OMG-RFTL} when run on sequence of matrices $\{\hat{A}_t \}_{t=1}^T$. The following two lemmas will allow us to do so. 

\begin{lemma}\label{e_to_x}
Let $\{e_{x,t}, e_{y,t}\}_{t=1}^T$ be the sequence of iterates generated by \textsc{Bandit-OMG-RFTL}. It holds that
\begin{align*} \textstyle
 \mathbb{E}\left[  \sum_{t=1}^T e_{x,t}^{\top}A_t e_{y,t}\right] =  \mathbb{E}\left[ \sum_{t=1}^T x^{\top}_t A_t y_t\right],
 \end{align*}
 where the expectation is taken with respect to the internal randomness of the algorithm.
 \end{lemma}

\begin{lemma}\label{A_hat_no_hat}
It holds that
\begin{align*} \textstyle
\mathbb{E}\left[\sum_{t=1}^T x_t^{\top} \hat{A}_t y_t\right] = \mathbb{E}\left[\sum_{t=1}^T x_t^{\top} A_t y_t\right],
\end{align*}
where the expectation is with respect to all the internal randomness of the algorithm.
\end{lemma}

We will then bound the difference between the comparator term $\MMD \sum_{t=1}^T x^{\top} A_t y$ and the comparator term Theorem \ref{thm:omg_rftl_regret} gives us by running \textsc{OMG-RFTL} on functions $\{\hat{\EL}\}_{t=1}^T$, $\MMD \sum_{t=1}^T x^{\top} \hat{A}_t y$. Special care must be taken to ensure this difference holds even against an adaptive adversary. To this end, we use the next two lemmas; in particular, the proof of Lemma \ref{lemma_with_alphas} relies heavily on Theorem \ref{thm:hess_estimate}.

\begin{lemma}\label{close_sp_vals}
With probability 1 it holds that
\begin{align*}
&\left| \min_{x\in \Delta_X^\delta}\max_{y\in \Delta_Y^\delta} \sum_{t=1}^T x^{\top}A_t y - \min_{x\in \Delta_X^\delta}\max_{y\in \Delta_Y^\delta} \sum_{t=1}^T x^{\top}\hat{A}_t y \right| \\
\leq & \textstyle \left \Vert \sum_{t=1}^T A_t y - \hat{A}_t y \right \Vert_2.
\end{align*}
\end{lemma}

\begin{lemma}\label{lemma_with_alphas}
It holds that
\begin{align*}
\mathbb{E}\left[ \left \Vert \sum_{t=1}^T A_t y - \hat{A}_t y \right \Vert_2 \right] \leq \frac{2 \sqrt{T} \min (d_1,d_2)}{\delta^2},
\end{align*}
where the expectation is taken with respect to the internal randomness of the algorithm.
\end{lemma}

 The proof of Theorem~\ref{no_bandit_regret} follows by combining Lemmas \ref{e_to_x} through \ref{lemma_with_alphas}, with careful choice of tuning parameters.
 
 \section{Training Generative Adversarial Networks}\label{sec:GANs}
 In this section we use our ideas to train Generative Adversarial Networks (GANs) \cite{goodfellow2014generative}.
 
 \subsection{GAN Formulation}
GANs are particular approach to generative modeling. A \textit{generative model} is is a machine learning model that takes samples drawn from an unknown distribution $p_{data}$ and learns to represent an estimate of that distribution. After training, the model outputs a distribution $p_{model}$ or some way to generate samples from it \cite{goodfellow2016nips}. A GAN can be though of as two neural networks, the \textit{generator} $G$ and the \textit{discriminator} $D$, playing a game against each other. The goal of the generator is to create samples from $p_{model}$ that look like samples from $p_{data}$, and the goal of the discriminator is to recognize if a given sample comes from $p_{data}$ or if it is a fake sample generated by its adversary. The original GAN formulation from \cite{goodfellow2014generative} poses the problem as finding a solution to 

\begin{equation}\label{eq:gan_original}
\min_{G} \max_{D} \mathbb{E}_{\bm{x} \sim p_{data}(\bm{x})} [\log(D(\bm{x}))] + \mathbb{E}_{z\sim p_{\bm{z}}(\bm{z})}[\log(1 - D(G(\bm{z})))].
\end{equation}

Here $p_{\bm{z}}(\bm{z})$ is some noise distribution that $G$ maps onto the data space. Generative models have plenty of applications in other areas of machine learning, for example: reinforcement learning \cite{finn2016unsupervised}, semi-supervised learning \cite{springenberg2015unsupervised,salimans2016improved}, single image super resolution \cite{ledig2017photo}, image-to-image translation \cite{isola2017image}, and even art creation \cite{brock2016neural}, just to mention a few. 

\subsection{Mode Collapse}
The most natural approach to train a GAN (and the original one used in \cite{goodfellow2014generative}), is to simultaneously perform gradient descent on the generator's parameters and gradient ascent on those of the discriminator.
However, it has been shown that even in simple convex-concave games such as $\EL(x,y)=xy$, if one performs gradient descent on $x$ and gradient ascent on $y$ the dynamics do not necessarily converge to the Nash Equilibrium (see Ch. 5 of \cite{goodfellow2016nips} ). So it should not be surprising to observe that serious problems arise while training a GAN. We say a GAN suffered from mode collapse if the generator ends up producing samples from only a few modes from the distribution $p_{data}$, visually it means that the generator produces samples with low diversity. The first row in Figure \ref{fig:alg_comparison} shows a clear example of this. 

Since the introduction of GANs there has been an incredible effort from the machine learning community to understand why mode collapse occurs and how to fix it. In a very recent large-scale study \cite{lucic2017gans}, many GAN models were thoroughly tested to see if one outperformed the others. Their conclusion was ``we did not find evidence that any of the tested algorithms consistently outperforms the non-saturating GAN introduced in \cite{goodfellow2014generative}". The algorithms/models tested in the aforementioned study were: MM-GAN \cite{goodfellow2014generative}, NS-GAN \cite{goodfellow2014generative}, WGAN \cite{arjovsky2017wasserstein}, WGAN GP \cite{gulrajani2017improved}, LS GAN \cite{mao2017least}, DRAGAN \cite{kodali2017convergence} and BEGAN \cite{berthelot2017began}. 

The algorithm for training the non-saturating GAN from \cite{goodfellow2014generative} corresponds to running two sublinear individual regret algorithms in parallel, one for the generator and another for the discriminator. However, it is common to observe mode collapse using this training procedure. In view of Theorem \ref{thm:choose_one_omg} we tested a variant of  \textsc{SP-RFTL} on this setting hoping for significantly different training dynamics.

\subsection{SP-RFTL for Training GANs}
Even though the original GAN formulation, Equation \ref{eq:gan_original}, is not convex-concave we tested a variant of \textsc{SP-RFTL} on this setting. The particular implementation consists on 1) taking a mini-batch of data from $p_{data}$ with $N$ samples to approximate the payoff function in Equation  \ref{eq:gan_original} with
\begin{align*}
\EL_t(G,D) = \frac{1}{N} \sum_{n=1}^N \log(D(\bm{x_n})) + \mathbb{E}_{z\sim p_{\bm{z}}(\bm{z})}[\log(1 - D(G(\bm{z})))],
\end{align*}
 2) simultaneously run a sublinear individual Regret algorithm on the generator's parameters and another one on the discriminator's parameters for a fixed number of iterations and 3) uniformly average the iterates of both algorithms. Then we sample a new mini-batch of data from $p_{data}$ to obtain $\EL_{t+1}$ and repeat the procedure. If the payoff function $\EL_t$ were convex-concave then the combination of steps 2) and 3) would be equivalent to finding an approximate NE for $\EL_t$. It is easy to see that the procedure just described follows the spirit of \textsc{SP-RFTL}  where both $R_X, R_Y$ are set to any constant function. The reason for this is that the sequence $\EL_t$ is stochastic (not adversarial) and thus regularization is not necessary. 
 
 \subsection{Experiments}
 In Figure \ref{fig:alg_comparison} we compare our proposed algorithm \textsc{SP-RFTL} with: Unrolled GAN \cite{metz2016unrolled}, Wassertein GAN \cite{arjovsky2017wasserstein}, and Wassertein GAN with Gradient Clipping \cite{gulrajani2017improved}. The dataset is a mixture of eight gaussians placed uniformly in a circle of radius two with variance .02. The generator and discriminator architectures for Unrolled 0, Unrolled 4, and \textsc{SP-RFTL}  are identical to those in Appendix A from \cite{metz2016unrolled}. The optimization parameters for Unrolled 0 and Unrolled 4 are the ones suggested in \cite{metz2016unrolled}. The optimization parameters for \textsc{SP-RFTL}  are the same as for Unrolled 0, the extra parameter that controls how often we average the iterates was tuned by visual inspection. The WGAN and WGAN-GC architectures and parameters are exactly the ones provided in \cite{gulrajani2017improved}. WGAN and WGAN-GC use an extra fully connected  hidden layer compared to Unrolled 0, Unrolled 4, and \textsc{SP-RFTL}, we did not change the architecture  assuming \cite{gulrajani2017improved} did their best effort to produce their original results. All the algorithms use Adam \cite{kingma2014adam,reddi2018convergence} as an optimizer. 
All the experiments were run on a Mac-Book Pro with processor 3.1 GHz Intel Core i7, and 16 GB of RAM. In particular, no GPU was used. All the code for this project can be found at https://github.com/adrianriv/gans-mode-collapse.

We judge the performance of the generator based on the quality of its samples. From Figure \ref{fig:alg_comparison} it is obvious that \textsc{SP-RFTL} learns the correct underlying distribution in the shortest amount of time. A final remark is that Unrolled 0 corresponds to running two no-individual regret algorithms in parallel, which results in mode collapse. Interestingly, \textsc{SP-RFTL} the algorithm with best performance, is doing exactly the same with the difference that it is averaging its iterates every fixed number of rounds. 

\begin{figure*}[htb]
  	\includegraphics[width=.5\linewidth]{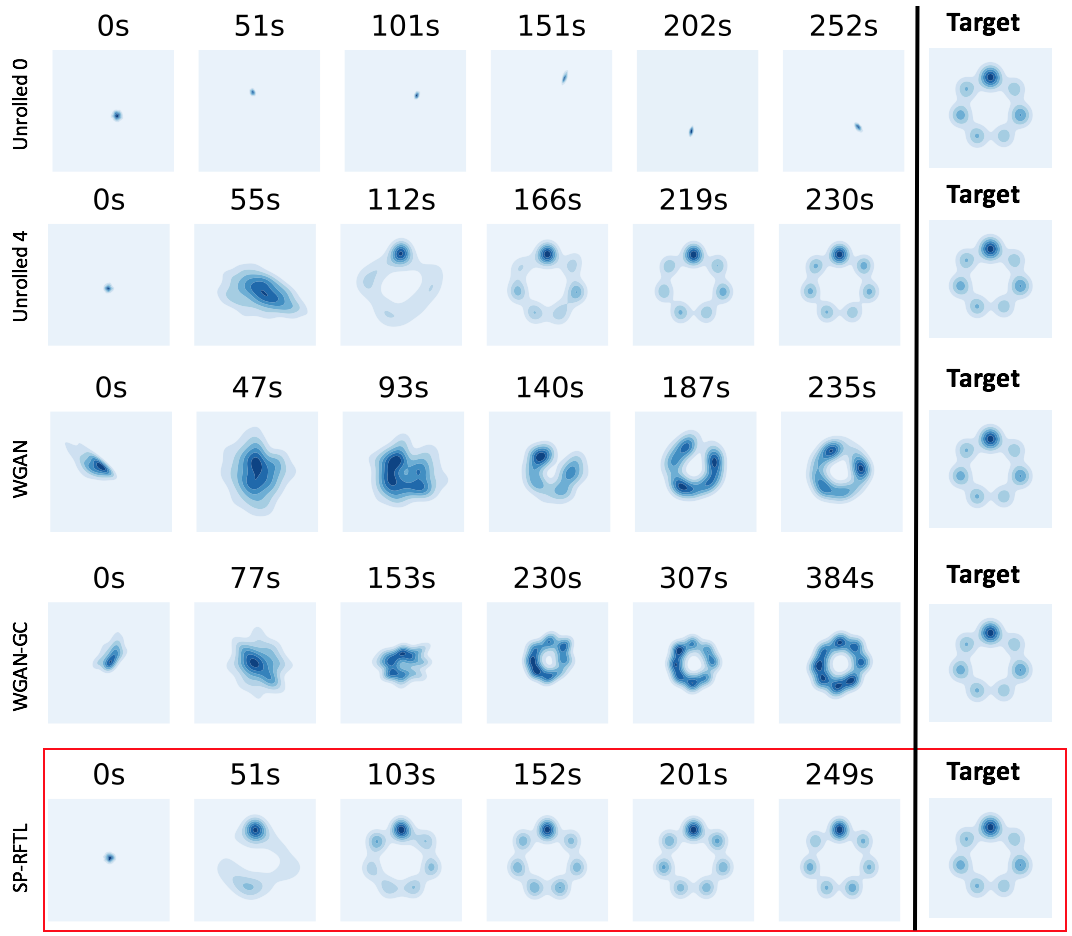}
  \centering
  \caption{Comparison of algorithms in the mixture of 8 gaussians dataset. Each image shows the probability density produced by the generator after $x$ seconds (CPU time) of training. It is clear that \textsc{SP-RFTL} (in red) outperforms all other algorithms.}
  \label{fig:alg_comparison}
\end{figure*}
 
\section{Conclusion}
In this paper, we considered the Online Matrix Games problem, where two players interact in a sequence of zero-sum games with arbitrarily changing payoff matrices. 
The goal for both players is to achieve small Nash Equilibrium (NE) regret, that is, the players want to ensure their average payoffs over $T$ rounds are close to those in the NE of the mean payoff matrix in hindsight. While it is known that standard Online Convex Optimization algorithms such as Online Gradient Descent can be used to find approximate equilibria in \emph{static} zero-sum games, our impossibility result shows that no algorithm for online convex optimization can achieve sublinear Nash Equilibrium regret ($o(T)$) when the sequence of payoffs are chosen arbitrarily. We then design and analyze algorithms that achieve sublinear NE regret for the Online Matrix Games problem, under both full information feedback and bandit feeback settings.  In the full information case, the performance of the algorithm is optimal with respect to the number of rounds (up to logarithm factors) and depends logarithmically on the number of actions of each player. For the bandit feedback setting, we provide an algorithm with sublinear NE regret using a one-point matrix estimate. Lastly, we test our algorithm for training GANs on a basic setup and obtain satisfactory results.

\section{Acknowledgements}
We thank an anonymous ICML reviewer who noticed an issue with our initial one-point estimate estimate of matrix $A_t$ and his/her suggestion on how to fix it.

\bibliography{mybib}
\bibliographystyle{abbrv} 
\newpage

\begin{appendix}

\section{Omitted Proofs}
\begin{proof}[Proof of Theorem \ref{thm:choose_one_omg}]
We assume there exists an algorithm such that 

$\vert \sum_{t=1}^Tx_t^\top A_t y_t - \MM  \sum_{t=1}^Tx^\top A_t y\vert \leq o(T)$, $\sum_{t=1}^Tx_t^\top A_t y_t - \min_{x\in \Delta_X} \sum_{t=1}^T x^\top A_t y_t \leq o(T)$, and $\max_{y\in \Delta_Y} \sum_{t=1}^T x_t^\top A_T y - \sum_{t=1}^Tx_t^\top A_t y_t \leq o(T)$ for all possible sequences of matrices $\{A_t\}_{t=1}^T$ with bounded entries between $[-1,1]$. We now construct two sequences of functions for which all the three guarantees hold and lead that to a contradiction. Let $T$ be divisible by $2$. In scenario 1: $A_t = 
\begin{bmatrix}
1 &-1 \\
-1 &1\\
\end{bmatrix}$
for $1\leq t \leq \frac{T}{2}$ and $A_t=
\begin{bmatrix}
0 &0 \\
0 &0\\
\end{bmatrix}$
for $\frac{T}{2}<t\leq T$.
In scenario 2: $A_t = 
\begin{bmatrix}
1 &-1 \\
-1 &1\\
\end{bmatrix}$
for $1\leq t \leq \frac{T}{2}$ and $A_t=
\begin{bmatrix}
1 &-1 \\
1 &-1\\
\end{bmatrix}$
for $\frac{T}{2}<t\leq T$. It is easy to see that for both scenarios it holds that $\MM \sum_{t=1}^T x^\top A_t y = 0$.
Since $d_1=d_2=2$ and we can parametrize any $x\in \Delta_X$ as $x = [\alpha; 1-\alpha]$ and any $y\in \Delta_Y$ as $y=[\beta;1-\beta]$ for some $0\leq \alpha, \beta \leq 1$. By assumption we have that $\max_{y\in \Delta_Y} \sum_{t=1}^T x_t^\top A_t y - \MM \sum_{t=1}^T x^\top A_t y \leq o(T)$ for all sequences of matrices $\{A_t\}_{t=1}^T$. This implies for scenario 1 that $\max_{0\leq \beta \leq 1} \sum_{t=1}^{\frac{T}{2}}4\alpha_t\beta - 2\beta + 1 -2\alpha_t \leq o(T)$ which also implies that $\sum_{t=1}^{\frac{T}{2}}2\alpha_t -1 \leq o(T)$ and $\sum_{t=1}^{\frac{T}{2}} 1-2\alpha_t \leq o(T)$ since $\sum_{t=1}^{\frac{T}{2}}4\alpha_t\beta - 2\beta + 1 -2\alpha_t$ is a linear function of $\beta$ and thus its maximum occurs at $\beta=0$ or $\beta=1$.

For scenario 2 $\max_{y\in \Delta_Y} \sum_{t=1}^T x_t^\top A_t y - \MM \sum_{t=1}^T x^\top A_t y \leq o(T)$ reduces to

 $\max_{0\leq \beta \leq 1} \sum_{t=1}^{\frac{T}{2}} 4\alpha_t \beta - 2\beta + 1 - 2\alpha_t + \frac{T}{2} (2\beta-1)\leq o(T)$ which implies $\sum_{t=1}^{\frac{T}{2}} 2\alpha_t -1 + \frac{T}{2} \leq o(T)$ and $\sum_{t=1}^{\frac{T}{2}} 1 -2\alpha_t  + \frac{T}{2} \leq o(T)$. Finally, notice that $\sum_{t=1}^{\frac{T}{2}} 2\alpha_t -1 + \frac{T}{2} \leq o(T)$ implies $\frac{T}{2} \leq o(T) +\sum_{t=1}^{\frac{T}{2}} 1 - 2\alpha_t$ but from scenario 1 we have that $\sum_{t=1}^{\frac{T}{2}} 1 - 2\alpha_t \leq o(T)$ since $\frac{T}{2}\leq o(T)$ is a contradiction we get the result.
\end{proof}

\begin{proof}[Proof of Lemma \ref{bilinear_lipschitz}]
We omit the subscript $t$.
\begin{align*}
\Vert \nabla x^\top Ay \Vert_2 &= 
\left\Vert
\begin{bmatrix}
\nabla_x x^\top Ay  \\
\nabla_y x^\top Ay
\end{bmatrix} \right\Vert_2 \\
&=
\left\Vert \begin{bmatrix}
A_{[1,:]}^\top y  \\
 ... \\
 A_{[d_1,:]}^\top y\\
 A_{[:,1]}^\top x  \\
 ... \\
 A_{[:,d_2]}^\top x
\end{bmatrix} \right\Vert_2 \\
& \leq 
\left\Vert \begin{bmatrix}
A_{[1,:]}^\top y  \\
 ... \\
 A_{[d_1,:]}^\top y\\
\end{bmatrix} \right\Vert_2 
+
\left\Vert \begin{bmatrix}
 A_{[:,1]}^\top x  \\
 ... \\
 A_{[:,d_2]}^\top x
\end{bmatrix} \right\Vert_2 \\
&\leq
\sqrt{\sum_{i=1}^{d_1}(A_{[i,:]}^\top y)^2} 
+
 \left\Vert \begin{bmatrix}
 A_{[:,1]}^\top x  \\
 ... \\
 A_{[:,d_2]}^\top x
\end{bmatrix} \right\Vert_2 \\
&\leq
\sqrt{d_1(\Vert A_{[i,:]}\Vert_{\infty} \Vert y\Vert_1)^2} 
+
\left\Vert \begin{bmatrix}
 A_{[:,1]}^\top x  \\
 ... \\
 A_{[:,d_2]}^\top x
\end{bmatrix} \right\Vert_2 \quad \text{by Generalized Cauchy Schwartz}\\
&\leq
\sqrt{c d_1} 
+
 \left\Vert \begin{bmatrix}
 A_{[:,1]}^\top x  \\
 ... \\
 A_{[:,d_2]}^\top x
\end{bmatrix} \right\Vert_2 \\
&\leq
\sqrt{c d_1} 
+ 
\sqrt{c d_2}. \quad \text{(using the same reasoning)}
\end{align*}
 The second part of the claim follows by bounding $\Vert \nabla x^\top Ay \Vert_\infty$ using the same argument.
\end{proof}

\begin{proof}[Proof of Lemma \ref{lemma:mm_bar_not_bar}]
\begin{align*}
\MM \sum_{t=1}^T \EL_t(x,y)&= \sum_{t=1}^T[\bar{\EL}_t(x_{T+1},y_{T+1}) + \frac{1}{\eta}R_X(x_{T+1}) - \frac{1}{\eta} R_Y(y_{T+1}) ]\\
&\leq \sum_{t=1}^T[\bar{\EL}_t(\bar{x}_{T+1},y_{T+1}) + \frac{1}{\eta}R_X(\bar{x}_{T+1}) - \frac{1}{\eta} R_Y(y_{T+1}) ]  & \text{by Equation \eqref{def_sp}} \\
&\leq \sum_{t=1}^T[\bar{\EL}_t(\bar{x}_{T+1},\bar{y}_{T+1}) + \frac{1}{\eta}R_X(\bar{x}_{T+1}) - \frac{1}{\eta} R_Y(y_{T+1}) ]  & \text{by Equation \eqref{def_sp}} \\
&= \MM  \sum_{t=1}^T[\bar{\EL}_t(x,y) + \frac{T}{\eta}R_X(\bar{x}_{T+1}) - \frac{T}{\eta} R_Y(y_{T+1}) ]\\
& \leq \MM  \sum_{t=1}^T \bar{\EL}_t(x,y) + \frac{T}{\eta}R_X(\bar{x}_{T+1}).
\end{align*}
The other inequality can be obtained by a similar argument.
\end{proof}

\begin{proof}[Proof of Lemma~\ref{lemma:loss_BTL}]
We first prove the second inequality. We proceed by induction. The base case $t=1$ holds by definition of $(x_2,y_2)$, indeed
\begin{align*}
\EL_1(x_2,y_2) + G_\EL \Vert y_1-y_2\Vert &\geq \EL_1(x_2,y_2) := \MM \EL_1(x,y). 
\end{align*} 
We now assume the following claim holds for $T-1$:
\begin{equation}\label{ind_hypo}
\MM \sum_{t=1}^{T-1} \EL_t(x,y)  \geq \sum_{t=1}^{T-1} \EL_t(x_{t+1},y_{t+1}) - G_\EL \sum_{t=1}^{T-1} \Vert y_t-y_{t+1}\Vert,
\end{equation}
and show it holds for $T$. 
\begin{align*}
&\MM \sum_{t=1}^T \EL_t(x,y)\\
& = \sum_{t=1}^{T-1} \EL_t(x_{T+1},y_{T+1}) + \EL_T(x_{T+1},y_{T+1})\\
&\geq \sum_{t=1}^{T-1} \EL_t(x_{T+1},y_T) + \EL_{T}(x_{T+1},y_T)\qquad \qquad\text{by Equation \eqref{def_sp}}\\
& \geq \sum_{t=1}^{T-1} \EL_t(x_T,y_T) + \EL_{T}(x_{T+1},y_T)\qquad \qquad \text{by Equation \eqref{def_sp}}\\
& \geq \sum_{t=1}^{T-1} \EL_t(x_{t+1},y_{t+1}) - G_\EL \sum_{t=1}^{T-1} \Vert y_t-y_{t+1}\Vert + \EL_{T}(x_{T+1},y_T) \qquad \text{by Equation \eqref{ind_hypo}}\\
&= \sum_{t=1}^{T} \EL_t(x_{t+1},y_{t+1})- G_\EL \sum_{t=1}^{T-1} \Vert y_t-y_{t+1}\Vert+ \EL_{T} (x_{T+1},y_T) - \EL_{T}(x_{T+1},y_{T+1})\\
&\geq \sum_{t=1}^{T} \EL_t(x_{t+1},y_{t+1})- G_\EL \sum_{t=1}^{T-1} \Vert y_t-y_{t+1}\Vert - G_\EL \Vert y_T - y_{T+1}\Vert\\
& = \sum_{t=1}^{T} \EL_t(x_{t+1},y_{t+1})- G_\EL \sum_{t=1}^{T} \Vert y_t-y_{t+1}\Vert.
\end{align*}

We now show by induction that 
\begin{align*}
\MM \sum_{t=1}^T \EL_t(x,y) \leq \sum_{t=1}^T \EL_t(x_{t+1},y_{t+1}) + G_\EL \sum_{t=1}^T \Vert x_t-x_{t+1}\Vert.   
\end{align*}
Indeed, $t=1$ follows from the definition of $(x_2,y_2)$. We now assume the claim holds for $T-1$ and prove it for $T$:
\begin{align*}
& \MM \sum_{t=1}^T \EL_t(x,y)\\
&= \sum_{t=1}^T \EL_t(x_{T+1},y_{T+1})\\
 &\leq \sum_{t=1}^{T-1} \EL_t(x_{T},y_{T+1}) + \EL_T(x_T, y_{T+1}) & \text{by Equation \eqref{def_sp}}\\
 &\leq \sum_{t=1}^{T-1} \EL_t(x_{T},y_{T}) + \EL_T(x_T, y_{T+1})& \text{by Equation \eqref{def_sp}}\\
 &\leq \sum_{t=1}^{T-1} \EL_t(x_{t+1},y_{t+1}) + G_\EL \sum_{t=1}^{T-1}\Vert x_{t}-x_{t+1}\Vert \\
 & \qquad + \EL_T(x_T, y_{T+1}) + \EL_T(x_{T+1},y_{T+1}) - \EL_T(x_{T+1},y_{T+1}) & \text{by induction claim} \\
 &\leq  \sum_{t=1}^{T} \EL_t(x_{t+1},y_{t+1}) + G_\EL \sum_{t=1}^{T}\Vert x_{t}-x_{t+1}\Vert & \text{since $\EL_T$ is $G_\EL$-Lipschitz}. 
 \end{align*}
\end{proof}

\begin{proof}[Proof of Lemma \ref{lemma:loss_BTL}]
Fix $t$, define $J(x,y)\triangleq \sum_{\tau=1}^{t-1} \EL_\tau(x,y) + \EL_t(x,y)$ and notice it is $\frac{t}{\eta}$-strongly convex strongly concave with respect to norm $ \Vert \cdot \Vert $. Also notice that $(x_{t+1},y_{t+1})$ is the unique saddle point of $J(x,y)$.

By strong convexity of $J$ and definition of $x_{t+1}$ it holds that for any $x\in X$ and any $y\in Y$
\begin{align*}
J(x,y) \geq J(x_{t+1},y) + \nabla_x J(x_{t+1},y)^{\top}(x - x_{t+1})  + \frac{t}{2 \eta}\Vert x-x_{t+1}\Vert^2.
\end{align*}
Plugging in $y=y_{t+1}$ and recalling the KKT condition $\nabla_x J(x_{t+1},y_{t+1})^{\top}(x - x_{t+1}) \geq 0$, we have that for any $x\in X$
\begin{equation}\label{J_KKT_convex}
\frac{2\eta}{t}\big[ J(x,y_{t+1}) - J(x_{t+1},y_{t+1})\big] \geq \Vert x - x_{t+1}\Vert^2.
\end{equation} 
Similarly, since $J$ is $\frac{t}{\eta}$ strongly concave. That is, for any $y\in Y$
\begin{align*}
J(x_{t+1},y) \leq J(x_{t+1},y_{t+1}) + \nabla_y J(x_{t+1},y_{t+1})^{\top}(y-y_{t+1}) - \frac{t}{2\eta} \Vert y-y_{t+1}\Vert^2.
\end{align*}
Together with the KKT condition $\nabla_y J(x_{t+1},y_{t+1})^{\top}(y-y_{t+1})\leq 0$ we get that for any $y\in Y$
\begin{equation} \label{J_KKT_concave}
\frac{2\eta}{t}\big[ J(x_{t+1},y_{t+1}) - J(x_{t+1},y)\big] \geq \Vert y-y_{t+1}\Vert^2.
\end{equation}
Adding up Equations \eqref{J_KKT_convex} and \eqref{J_KKT_concave}, plugging $x=x_t$ and $y=y_t$ we get
\begin{align*}
& \frac{2\eta}{t} \big[ J(x_t,y_{t+1})-J(x_{t+1},y_t) \big] \geq \Vert x_t - x_{t+1}\Vert^2 + \Vert y-y_{t+1}\Vert^2.
 \\
\iff & \frac{2\eta}{t} \big[ \sum_{\tau=1}^{t-1}\EL_\tau (x_t,y_{t+1}) + \EL_t (x_t,y_{t+1}) - [ \sum_{\tau=1}^{t-1} \EL_{\tau}(x_{t+1},y_t) + \EL_t (x_{t+1},y_t)]\big]  \\
& \qquad \geq \Vert x_t - x_{t+1}\Vert^2 + \Vert y-y_{t+1}\Vert^2 \\
 \implies &\frac{2\eta}{t} \big[ \sum_{\tau=1}^{t-1}\EL_\tau (x_t,y_{t}) + \EL_t (x_t,y_{t+1}) - [ \sum_{\tau=1}^{t-1} \EL_{\tau}(x_{t+1},y_t) + \EL_t (x_{t+1},y_t)]\big] \\
 & \qquad \geq \Vert x_t - x_{t+1}\Vert^2 + \Vert y-y_{t+1}\Vert^2,
\end{align*}
since $\sum_{\tau=1}^{t-1}\EL_\tau (x_t,y_{t+1}) \leq \sum_{\tau=1}^{t-1}\EL_\tau (x_t,y_{t})$. 

Additionally, since $\sum_{\tau=1}^{t-1}\EL_\tau (x_t,y_{t}) \leq \sum_{\tau=1}^{t-1}\EL_\tau (x_{t+1},y_{t})$, we have
\begin{align*}
&\frac{2\eta}{t} \big[ \sum_{\tau=1}^{t-1}\EL_\tau (x_t,y_{t}) + \EL_t (x_t,y_{t+1}) -  \sum_{\tau=1}^{t-1} \EL_{\tau}(x_{t},y_t) - \EL_t (x_{t+1},y_t) \big]\\
& \qquad  \geq \Vert x_t - x_{t+1}\Vert^2 + \Vert y-y_{t+1}\Vert^2\\
\iff & \frac{2\eta}{t} \big[  \EL_t (x_t,y_{t+1}) - \EL_t (x_{t+1},y_t) \big]  \geq \Vert x_t - x_{t+1}\Vert^2 + \Vert y_t-y_{t+1}\Vert^2\\
\iff & \frac{2\eta}{t} \big[ \bar{\EL}_t(x_{t},y_{t+1}) + \frac{1}{\eta} R_X(x_{t})- \frac{1}{\eta}R_Y(y_{t+1}) - \bar{\EL}_t(x_{t+1},y_t)-\frac{1}{\eta}R_X(x_{t+1}) + \frac{1}{\eta}R_Y(y_t) \big]\\
& \qquad \geq \Vert x_t - x_{t+1}\Vert^2 + \Vert y_t-y_{t+1}\Vert^2\\
\implies & \frac{2\eta}{t} \big[G_{\bar{\EL}} \Vert [x_t;y_{t+1}] - [x_{t+1};y_{t}] \Vert  + \frac{1}{\eta} R_X(x_{t}) -\frac{1}{\eta}R_X(x_{t+1}) + \frac{1}{\eta}R_Y(y_t) - \frac{1}{\eta}R_Y(y_{t+1}) \big]\\
& \qquad \geq \Vert x_t - x_{t+1}\Vert^2 + \Vert y_t-y_{t+1}\Vert^2\\
\implies & \frac{2\eta}{t} \big[G_{\bar{\EL}} \Vert x_t - x_{t+1} \Vert +G_{\bar{\EL}} \Vert y_t - y_{t+1} \Vert   + \frac{1}{\eta} R_X(x_{t}) -\frac{1}{\eta}R_X(x_{t+1}) + \frac{1}{\eta}R_Y(y_t) - \frac{1}{\eta}R_Y(y_{t+1}) \big]\\
& \qquad \geq \Vert x_t - x_{t+1}\Vert^2 + \Vert y_t-y_{t+1}\Vert^2\\
\implies & \frac{2\eta}{t} \big[G_{\bar{\EL}} \Vert x_t - x_{t+1} \Vert +G_{\bar{\EL}} \Vert y_t - y_{t+1} \Vert   + \frac{G_{R_X}}{\eta} \Vert x_t - x_{t+1}\Vert + \frac{G_{R_Y}}{\eta} \Vert y_{t}- y_{t+1} \Vert \big]\\
& \qquad \geq \Vert x_t - x_{t+1}\Vert^2 + \Vert y_t-y_{t+1}\Vert^2\\
\implies & \frac{2\eta}{t} [G_{\bar{\EL}}+\frac{1}{\eta}\max(G_{R_X}, G_{R_Y})] \big[\Vert x_t - x_{t+1} \Vert +\Vert y_t - y_{t+1} \Vert \big] \geq \Vert x_t - x_{t+1}\Vert^2 + \Vert y_t-y_{t+1}\Vert^2\\
\iff & \frac{2\eta}{t} [G_{\bar{\EL}}+\frac{1}{\eta}\max(G_{R_X}, G_{R_Y})] \geq \frac{\Vert x_t - x_{t+1}\Vert^2 + \Vert y_t-y_{t+1}\Vert^2}{\Vert x_t - x_{t+1} \Vert +\Vert y_t - y_{t+1} \Vert}.\\
\end{align*}
Finally, since $x^2$ is a convex function $\frac{a^2}{2}  + \frac{b^2}{2}  \geq \big( \frac{a+b}{2}\big)^2$, we have $a^2 + b^2 \geq \frac{(a+b)^2}{2}$. This, together with the last implication, yields the result
\begin{align*}
\frac{4 \eta }{t} [G_{\bar{\EL}}+ \frac{1}{\eta}\max(G_{R_X}, G_{R_Y})]  \geq \Vert x_{t}-x_{t+1} \Vert + \Vert y_t - y_{t+1}\Vert.
\end{align*}
\end{proof}

\begin{proof}[Proof of Theorem \ref{theorem:sp_regret_convex_concave}]
\begin{align*}
& \sum_{t=1}^T \bar{\EL}_t(x_t,y_t) - \MM \sum_{t=1}^T \bar{\EL}_t(x,y)\\
& \leq \sum_{t=1}^T \EL_t(x_t,y_t) - \MM \sum_{t=1}^T \bar{\EL}_t(x,y) + \sum_{t=1}^T \frac{1}{\eta} R_Y(y_t) \quad \text{by Equation \ref{eq:diff_bar_not_bar}}\\
& \leq \sum_{t=1}^T \EL_t(x_t,y_t) - \MM \sum_{t=1}^T \EL_t(x,y) + \sum_{t=1}^T \frac{1}{\eta} R_Y(y_t) + \frac{T}{\eta} R_X(x_{T+1})  \quad \text{by Lemma \ref{lemma:mm_bar_not_bar}}\\
& \leq \sum_{t=1}^T \EL_t(x_t,y_t) - \sum_{t=1}^T \EL_t(x_{t+1},y_{t+1}) + \sum_{t=1}^T \frac{1}{\eta} R_Y(y_t) + \frac{T}{\eta} R_X(x_{T+1}) + G_{\EL} \sum_{t=1}^T \Vert y_t- y_{t+1}\Vert  \quad \text{by Lemma \ref{lemma:loss_BTL}}\\
& \leq \sum_{t=1}^T G_{\EL} (\Vert x_t - x_{t+1}\Vert + \Vert y_{t}-y_{t+1}\Vert) + \sum_{t=1}^T \frac{1}{\eta} R_Y(y_t) + \frac{T}{\eta} R_X(x_{T+1}) + G_{\EL} \sum_{t=1}^T \Vert y_t- y_{t+1}\Vert \\
& \leq \sum_{t=1}^T G_{\EL} (\Vert x_t - x_{t+1}\Vert + \Vert y_{t}-y_{t+1}\Vert) + \sum_{t=1}^T \frac{1}{\eta} R_Y(y_t) + \frac{T}{\eta} R_X(x_{T+1}) + G_{\EL} \sum_{t=1}^T \Vert y_t- y_{t+1}\Vert \\
& \leq 2\sum_{t=1}^T G_{\EL} (\Vert x_t - x_{t+1}\Vert + \Vert y_{t}-y_{t+1}\Vert) + \sum_{t=1}^T \frac{1}{\eta} R_Y(y_t) + \frac{T}{\eta} R_X(x_{T+1}) \\
& \leq 2\sum_{t=1}^T G_{\EL} ( \frac{4 \eta }{t} [G_{\bar{\EL}}+ \frac{1}{\eta}\max(G_{R_X}, G_{R_Y})]
) + \sum_{t=1}^T \frac{1}{\eta} R_Y(y_t) + \frac{T}{\eta} R_X(x_{T+1}) \\
& \leq 8 G_{\EL} \eta [G_{\bar{\EL}}+ \frac{1}{\eta}\max(G_{R_X}, G_{R_Y})] ( 1 + \int_{1}^T \frac{1}{t} dt )+ \sum_{t=1}^T \frac{1}{\eta} R_Y(y_t) + \frac{T}{\eta} R_X(x_{T+1})\\
& \leq 8 G_{\EL}  \eta [G_{\bar{\EL}}+ \frac{1}{\eta}\max(G_{R_X}, G_{R_Y})] ( 1 + \ln(T) )+ \frac{T}{\eta} \max_{y\in Y} R_Y(y) +  \frac{T}{\eta} \max_{x\in X} R_X(x)\\
& \leq 8 \eta [G_{\bar{\EL}}+ \frac{1}{\eta}\max(G_{R_X}, G_{R_Y})]^2 ( 1 + \ln(T) )+ \frac{T}{\eta} \max_{y\in Y} R_Y(y) +  \frac{T}{\eta} \max_{x\in X} R_X(x).
\end{align*}
Notice that  $\MM \sum_{t=1}^T \bar{\EL}_t(x,y) - \sum_{t=1}^T \bar{\EL}_t(x_t,y_t)$ can be upper bounded by the same quantity using the same argument. This concludes the proof.
\end{proof}

\begin{proof}[Proof of Lemma \ref{lemma:entropy_lipschitz}]
We need to find $G_R>0$ such that $\Vert \nabla R(x)\Vert_\infty \leq G_R$ for all $x\in \Delta_\theta$. Notice that $[\nabla R(x)]_i = 1 + \ln(x_i)$ for $i=1,...d$. Moreover, since for every $i=1,...,d$ we have $\theta \leq x_i\leq 1$ the following sequence of inequalities hold: $\ln(\theta)\leq 1 +\ln(\theta) \leq 1+ \ln(x_i) \leq 1$. It follows that $G_R = \max\{|\ln(\theta)|,1\}$.
\end{proof}

\begin{proof}[Proof of Lemma \ref{lemma:dist_proj_sp}]
Choose $z^*=[1;0;0;...;0;0]$, it is easy to see that $z^*_p = [1-\theta (d-1); \theta; \theta; ...; \theta, \theta]$ and $\Vert z^* - z^*_p \Vert_1 = 2\theta(d-1).$ 
\end{proof}

\begin{proof}[Proof of Lemma \ref{lemma:sp_val_error_theta}]
Let $(x^*,y^*)$ be any saddle point pair for $\sum_{t=1}^T \bar{\EL}_t(x,y)$ with $x^*\in \Delta, y^*\in \Delta$.  Let $(x^*_\theta,y^*_\theta)$ be any saddle point pair for $\sum_{t=1}^T \bar{\EL}_t(x,y)$ with $x^*_\theta \in \Delta, y^*_\theta \in \Delta$. Let $x^*_p, y^*_p$ be the projection of $x^*, y^*$ onto the respective simplexes using the $\Vert \cdot \Vert_\infty$ norm. We first show the second inequality. Notice that 
\begin{align*}
\sum_{t=1}^T \bar{\EL}_t (x^*, y^*) & \leq \sum_{t=1}^T \bar{\EL}_t (x^*_\theta, y^*)\\
& \leq \sum_{t=1}^T \bar{\EL}_t (x^*_\theta, y^*_p) + G_{\bar{\EL}}T \Vert y^*_p - y^*\Vert_1\\
& \leq \sum_{t=1}^T \bar{\EL}_t (x^*_\theta, y^*_\theta) + G_{\bar{\EL}}T \Vert y^*_p - y^*\Vert_1.
\end{align*}
To show the first inequality in the statement of the lemma notice that 
\begin{align*}
\sum_{t=1}^T \bar{\EL}_t (x^*, y^*) & \geq \sum_{t=1}^T \bar{\EL}_t (x^*, y^*_\theta) \\
& \geq \sum_{t=1}^T \bar{\EL}_t (x^*_p, y^*_\theta) - G_{\bar{\EL}} T \Vert x^*_p - x^*\Vert_1\\
& \geq \sum_{t=1}^T \bar{\EL}_t(x^*_\theta, y^*_\theta) - G_{\bar{\EL}} T \Vert x^*_p - x^*\Vert_1.
\end{align*}
This concludes the proof.
\end{proof}

\begin{proof}[Proof of Theorem \ref{thm:omg_rftl_regret}]
 For convenience set $\bar{\EL}_t(x,y) = x^\top A_t y$. Let $(x^*, y^*)$ be any saddle point of $\MMD \sum_{t=1}^T x^\top A_t y$, let $(x^*_p, y^*_p)$ be the respective projections onto $\Delta_\theta$ using $\Vert \cdot \Vert_\infty$ norm. By the choice of $\theta$ we have that $|\ln(\theta)|>1$ additionally, notice that $\max_{z\in \Delta_\theta} \sum_{i=1}^d z_i \ln(z_i) + \ln(d) \leq 0 +\ln(d)$ by Jensen's inequality.
\begin{align*}
& \sum_{t=1}^T x_t^\top A_t y_t - \MMD \sum_{t=1}^T x^\top A_t y \\
& \leq \sum_{t=1}^T x_t^\top A_t y_t - \MMDt \sum_{t=1}^T x^\top A_t y + G_{\bar{\EL}} T \Vert x^* - x^*_p\Vert_1  \quad \text{by Lemma \ref{lemma:sp_val_error_theta} }\\
& \leq \sum_{t=1}^T x_t^\top A_t y_t - \MMDt \sum_{t=1}^T x^\top A_t y + 2G_{\bar{\EL}} T \theta (d_1-1)  \quad \text{by Lemma \ref{lemma:dist_proj_sp}}\\
&\leq  8 \eta [G_{\bar{\EL}}+ \frac{1}{\eta}\max(G_{R_X}, G_{R_Y})]^2 ( 1 + \ln(T) )+ \frac{T}{\eta} \max_{y\in \Delta_\theta} R_Y(y) +  \frac{T}{\eta} \max_{x\in \Delta_\theta} R_X(x) + 2G_{\bar{\EL}} T \theta (d_1-1)  \quad \text{by Theorem \ref{theorem:sp_regret_convex_concave}}\\
 &\leq  8 \eta [G_{\bar{\EL}}+ \frac{|\ln(\theta)|}{\eta}]^2 ( 1 + \ln(T) )+ \frac{T}{\eta} \max_{y\in \Delta_\theta} R_Y(y) +  \frac{T}{\eta} \max_{x\in \Delta_\theta} R_X(x) + 2G_{\bar{\EL}} T \theta (d_1-1)\\
 &\leq 32 \eta G_{\bar{\EL}}^2 (1+\ln(T)) + \frac{T}{\eta} \max_{y\in \Delta_\theta} R_Y(y) +  \frac{T}{\eta} \max_{x\in \Delta_\theta} R_X(x) + 2G_{\bar{\EL}} T e^{-\eta G_{\bar{\EL}}} (d_1-1) \quad \text{by the choice of $\theta$}\\
 & \leq 32 \eta G_{\bar{\EL}}^2 (1+\ln(T)) + \frac{T}{\eta}\ln(d_2) +  \frac{T}{\eta} \ln(d_1)+ 2G_{\bar{\EL}} T e^{-\eta G_{\bar{\EL}}} (d_1-1)\\
 & \leq 32 G_{\bar{\EL}} \sqrt{T} (1 + \ln(T)) + \sqrt{T} (\ln d_1+\ln d_2) + 2 d_1 G_{\bar{\EL}} T e^{-\sqrt{T}} \\
& = O\left(\ln(T)\sqrt{T} +  \sqrt{T} \max\{\ln d_1 ,\ln d_2\}\right) +\quad o(1)\max\{d_1,d_2\}.
\end{align*}
The last line follows because $G_{\bar{\EL}}\leq 1$, since each entry of A is bounded between $[-1, 1]$. A symmetrical argument yields the other side of the inequality. 
\end{proof}

\begin{proof}[Proof of Lemma \ref{A_hat_no_hat}]
\begin{align*}
&\mathbb{E}[\sum_{t=1}^T x_t^{\top} \hat{A}_t y_t]\\
& = \mathbb{E}[\sum_{t=1}^{T-1} x_t^{\top} \hat{A}_t y_t] + \mathbb{E}[x_T^{\top} \hat{A}_T y_T] \\
& = \mathbb{E}[\sum_{t=1}^{T-1} x_t^{\top} \hat{A}_t y_t] + \mathbb{E}[ \mathbb{E}[x_T^{\top} \hat{A}_T y_T|\tau=1,...,T-1]]\\
 & = \mathbb{E}[\sum_{t=1}^{T-1} x_t^{\top} \hat{A}_t y_t] + \mathbb{E}[ x_T^{\top} \mathbb{E}[ \hat{A}_T |\tau=1,...,T-1]y_T ]\\
  & = \mathbb{E}[\sum_{t=1}^{T-1} x_t^{\top} \hat{A}_t y_t] + \mathbb{E}[ x_T^{\top} A_T y_T ] \quad \text{by Theorem \ref{thm:hess_estimate}}.
\end{align*}
Repeating the argument $T-1$ more times yields the result.
\end{proof}

\begin{proof}[Proof of Lemma \ref{close_sp_vals}]
Let us fist bound $|\sum_{t=1}^T x^\top A_t y - \sum_{t=1}^T x^\top \hat{A}_t y|$  for any $x \in \Delta_X$ and $y \in \Delta_Y$ with probability 1.
\begin{align*}
& |\sum_{t=1}^T x^\top A_t y - \sum_{t=1}^T x^\top \hat{A}_t y|\\
& = | x^{\top} ( \sum_{t=1}^T A_t y - \sum_{t=1}^T \hat{A}_t y ) |\\
& \leq \Vert x\Vert_2 \Vert \sum_{t=1}^T A_t y - \hat{A}_t y \Vert_2\\
& \leq \Vert \sum_{t=1}^T A_t y - \hat{A}_t y \Vert_2
\end{align*}
It now follows that
\begin{align*}
& \sum_{t=1}^T x^{\top} \hat{A}_t y \leq \sum_{t=1}^T x^{\top}A_t y + \Vert \sum_{t=1}^T A_t y - \hat{A}_t y \Vert_2\\
\implies & \min_{x\in \Delta_{X,\delta}} \sum_{t=1}^T x^{\top} \hat{A}_t y \leq \sum_{t=1}^T x^{\top}A_t y + \Vert \sum_{t=1}^T A_t y - \hat{A}_t y \Vert_2 \quad \forall x \in \Delta_{X,\delta} ,y \in \Delta_{Y,\delta} \\
\implies & \min_{x\in \Delta_{X,\delta}} \sum_{t=1}^T x^{\top} \hat{A}_t y \leq \max_{y\in \Delta_{Y,\delta}} \sum_{t=1}^T x^{\top}A_t y + \Vert \sum_{t=1}^T A_t y - \hat{A}_t y \Vert_2 \quad \forall x \in \Delta_{X,\delta},y \in \Delta_{Y,\delta} \\
\implies & \max_{y\in \Delta_{Y,\delta}} \min_{x\in \Delta_{X,\delta}} \sum_{t=1}^T x^{\top} \hat{A}_t y \leq \min_{x\in \Delta_{X,\delta}} \max_{y\in \Delta_{Y,\delta}} \sum_{t=1}^T x^{\top}A_t y + \Vert \sum_{t=1}^T A_t y - \hat{A}_t y \Vert_2 \quad \forall x \in \Delta_{X,\delta},y \in \Delta_{Y,\delta}. 
\end{align*}
This concludes the proof as  $\max_{y\in \Delta_{Y,\delta}} \min_{x\in \Delta_{X,\delta}} \sum_{t=1}^T x^{\top} \hat{A}_t y = \min_{x\in \Delta_{X,\delta}} \max_{y\in \Delta_{Y,\delta}} \sum_{t=1}^T x^{\top} \hat{A}_t y$ (since the function is convex-concave and the sets $ \Delta_Y^\delta$ and $\Delta_X^\delta$ are convex and compact), the other side of the inequality can be obtained using the other inequality follows from applying the same reasoning.  
\end{proof}

\begin{proof}[Proof of Lemma \ref{lemma_with_alphas}]
For any $y$ define $\alpha_t \triangleq A_t y - \hat{A}_t y$.
We first show that for all $t, t'$ such that $t<t'$ it holds that $\mathbb{E}[\alpha_t^{\top}\alpha_{t'}]=0$. Indeed
\begin{align*}
\mathbb{E}[\alpha_t^{\top} \alpha_{t'}] &= \mathbb{E}[(A_t y -\hat{A}_t y )^{\top}(A_{t'} y -\hat{A}_{t'} y )]\\
&= \mathbb{E}[(A_t y)^{\top}A_{t'} y - (A_t y)^{\top}\hat{A}_{t'} y - (\hat{A}_t y )^{\top}A_{t'} y + (\hat{A}_t y )^{\top} \hat{A}_{t'} y ]\\
& = (A_t y)^{\top}A_{t'} y  - (A_t y)^{\top}A_{t'} y  - (A_t y )^{\top}A_{t'} y + \mathbb{E}[(\hat{A}_t y )^{\top} \hat{A}_{t'} y ]\\
& = (A_t y)^{\top}A_{t'} y  - (A_t y)^{\top}A_{t'} y  - (A_t y )^{\top}A_{t'} y + (A_t y )^{\top} A_{t'} y\\
&= 0,
\end{align*}
where the second to last line follows since 
\begin{align*}
\mathbb{E}[(\hat{A}_t y )^{\top} \hat{A}_{t'} y ] &= \mathbb{E}_{1,...,t'-1}[ \mathbb{E}[(\hat{A}_t y )^{\top} \hat{A}_{t'} y |\tau = 1,..., t'-1]]\\
&= \mathbb{E}_{1,...,t'-1}[ (\hat{A}_t y )^{\top} \mathbb{E}[ \hat{A}_{t'} y |\tau = 1,..., t'-1]]\\
&= \mathbb{E}_{1,...,t'-1}[ (\hat{A}_t y )^{\top} A_{t'}y ]\\
&= (A_t y )^{\top} A_{t'}y.
\end{align*}

Now,
\begin{align*}
\mathbb{E}[ \Vert \sum_{t=1}^T A_t y - \hat{A}_t y \Vert_2 ] & = \sqrt{\mathbb{E}[\Vert \sum_{t=1}^T \alpha_t \Vert_2 ]^2}\\
& \leq \sqrt{ \mathbb{E}[\Vert \sum_{t=1}^T \alpha_t \Vert_2^2 ]} \quad \text{by Jensen's Inequality}\\
& = \sqrt{ \sum_{t=1}^T \mathbb{E}[\Vert \alpha_t\Vert_2^2] + 2 \sum_{t < t'} \mathbb{E}[\alpha_t^{\top}\alpha_{t'}]}\\
& = \sqrt{\sum_{t=1}^T \mathbb{E}[\Vert A_t y - \hat{A}_t y\Vert_2^2]}\\
& \leq \sqrt{ \sum_{t=1}^T \mathbb{E}[2 \Vert A_t y\Vert ^2 + 2\Vert \hat{A}_t y\Vert_2^2 ] }
\end{align*}
We proceed to bound $\Vert \hat{A}_t y\Vert_2$, the upper bound we obtain will also bound $\Vert A_t y\Vert $ because of the following fact. If the random vector $\tilde{a}$ satisfies $\Vert \tilde{a}\Vert \leq c$ for some constant c with probability 1 then $\Vert \mathbb{E}\tilde{a}\Vert \leq c$. Indeed by Jensen's inequality we have that $\Vert \mathbb{E}\tilde{a}\Vert \leq \mathbb{E} \Vert \tilde{a}\Vert \leq c$. Let us omit the subscript $t$ for the rest of the proof. Let $\hat{A}_{[i,:]}$ be the $i$-th row of matrix $\hat{A}$.

\begin{align*}
\Vert \hat{A}y\Vert_2 &= \sqrt{\sum_{i=1}^{d_1} \big[\sum_{j=1}^{d_2} \hat{a}_{i,j} y_j\big]^2}\\
&\leq \sum_{i=1}^{d_1} \sqrt{\big[\sum_{j=1}^{d_2} \hat{a}_{i,j} y_j\big]^2}\\
&= \sum_{i=1}^{d_1} \big| \sum_{j=1}^{d_2} \hat{a}_{i,j} y_j \big| \\
& \leq \sum_{i=1}^{d_1} \Vert \hat{A}_{[i,:]} \Vert_{\infty} \Vert y\Vert_1 \quad \text{by generalized Cauchy Schwartz}\\
& \leq d_1 \max_{i,j} |\frac{A_{i,j}}{\delta^2}| \quad \text{by definition of $\hat{A}$ and using the fact that $ x_t \in \Delta_{X,\delta}$ and $y_t \in \Delta_{Y,\delta}$}\\
& \leq  \frac{d_1}{\delta^2}.
\end{align*}
Notice the upper bound $\frac{d_2}{\delta^2}$ can also be obtained by interchanging the summations and repeating the argument. This yields the desired result.
\end{proof}

\begin{proof}[Proof of Theorem \ref{no_bandit_regret}]
We first focus on one side of the inequality, 
\begin{align*}
& \mathbb{E}[  \sum_{t=1}^T e_{x,t}^{\top}A_t e_{y,t} - \MM  \sum_{t=1}^T x^{\top}A_t y ] \\
& =  \mathbb{E}[  \sum_{t=1}^T e_{x,t}^{\top}A_t e_{y,t}] - \mathbb{E}[\MM  \sum_{t=1}^T x^{\top}A_t y ]\\
& = \mathbb{E}[  \sum_{t=1}^T x_t^{\top}A_t y_t] - \mathbb{E}[\MM  \sum_{t=1}^T x^{\top}A_t y ] \quad \text{by Lemma \ref{e_to_x} }\\
& = \mathbb{E}[  \sum_{t=1}^T x_t^{\top}A_t y_t] - \mathbb{E}[ \min_{x\in \Delta_X^\delta}\max_{y\in \Delta_Y^\delta} \sum_{t=1}^T x^{\top}A_t y ] + 2\delta G_{\bar{\EL}}^{\Vert \cdot \Vert_1}(d_1-1) T \quad {\text{by Lemmas \ref{lemma:dist_proj_sp} and \ref{lemma:sp_val_error_theta} }}\\
& \leq \mathbb{E}[  \sum_{t=1}^T x_t^{\top}A_t y_t] - \mathbb{E}[ \min_{x\in \Delta_X^\delta}\max_{y\in \Delta_Y^\delta} \sum_{t=1}^T x^{\top}\hat{A}_t y ] + \frac{2 \sqrt{T}\min(d_1,d_2)}{\delta^2}  + 2\delta G_{\bar{\EL}}^{\Vert \cdot \Vert_1}(d_1-1) T \quad \text{by Lemmas \ref{close_sp_vals} and \ref{lemma_with_alphas}}\\
& \leq \mathbb{E}[  \sum_{t=1}^T x_t^{\top}\hat{A}_t y_t] - \mathbb{E}[ \min_{x\in \Delta_X^\delta}\max_{y\in \Delta_Y^\delta} \sum_{t=1}^T x^{\top}\hat{A}_t y ] + \frac{2 \sqrt{T}\min(d_1,d_2)}{\delta^2}   + 2\delta G_{\bar{\EL}}^{\Vert \cdot \Vert_1}(d_1-1) T \quad \text{by Lemma \ref{A_hat_no_hat}}\\
& \leq  8 \eta [G_{\hat{\EL}}^{\Vert \cdot \Vert_1} + \frac{\vert \ln(\delta)\vert}{\eta}]^2 (1+\ln(T)) +\frac{T}{\eta} (\ln(d_1)+ \ln(d_2))\\
&\quad + \frac{2 \sqrt{T}\min(d_1,d_2) }{\delta^2} + 2\delta G_{\bar{\EL}}^{\Vert \cdot \Vert_1}(d_1-1) T \quad \text{as in the proof of Theorem \ref{thm:omg_rftl_regret}}\\
& =  8 \eta [\frac{1}{\delta^2}+ \frac{\vert \ln(\delta)\vert}{\eta}]^2 (1+\ln(T)) +\frac{T}{\eta} (\ln(d_1)+ \ln(d_2)) + \frac{2 \sqrt{T} \min(d_1,d_2) }{\delta^2} + 2\delta (d_1-1) T \quad \text{by Lemma \ref{bilinear_lipschitz}}\\
& = O((d_1 + d_2) \ln(T) T^{5/6})\quad \text{ after plugging in $\delta = \frac{1}{T^{1/6}}$, $\eta = T^{1/6}$}
\end{align*}
The other side of the inequality follows by a symmetrical argument.
\end{proof}

\end{appendix}

\end{document}